%% file: main_file.tex
\documentclass[sigconf]{acmart} 

\copyrightyear{2026}
\acmYear{2026}
\setcopyright{cc}
\setcctype{by}
\acmConference[KDD '26]{Proceedings of the 32nd ACM SIGKDD Conference on Knowledge Discovery and Data Mining V.2}{August 09--13, 2026}{Jeju Island, Republic of Korea}
\acmBooktitle{Proceedings of the 32nd ACM SIGKDD Conference on Knowledge Discovery and Data Mining V.2 (KDD '26), August 09--13, 2026, Jeju Island, Republic of Korea}
\acmDOI{10.1145/3770855.3817966}
\acmISBN{979-8-4007-2259-2/2026/08}

\usepackage{amsfonts}

\usepackage{amssymb}
\usepackage{bm}
\usepackage{nicefrac}
\usepackage{amsmath}
\usepackage{multirow}

\usepackage{subcaption} % Confirmed as APPROVED
\usepackage{enumitem}   % Confirmed as APPROVED

\usepackage{xspace}

\newcommand{\nop}[1]{}

\expandafter\def\expandafter\normalsize\expandafter{%
    \normalsize%
    \setlength\abovedisplayskip{0pt}%
    \setlength\belowdisplayskip{0pt}%
    \setlength\abovedisplayshortskip{0pt}%
    \setlength\belowdisplayshortskip{0pt}%
}

\input{others/packages}

\input{others/configs}

\begin{document}

% needs iteration:

\title[Gradient Orthogonal Low-Rank Adaptation Framework for Graph Continual Learning]{G$^2$LoRA: Gradient Orthogonal Low-Rank Adaptation Framework for Graph Continual Learning on Text-Attributed Graphs}
\input{others/authors}

\author{Yuhan Wang}
\affiliation{%
  \institution{School of Computer Science and
Engineering, Beihang University}
  \city{Beijing}
  \country{China}
}
\email{wyh7667@gmail.com}

\author{Yibo Ding}
\affiliation{%
  \institution{School of Computer Science and
Engineering, Beihang University}
  \city{Beijing}
  \country{China}
}
\email{dingyibo@buaa.edu.cn}

\author{Yutong Ye}
\affiliation{%
  \institution{School of Computer Science and
Engineering, Beihang University}
  \city{Beijing}
  \country{China}
}
\email{yutongye@buaa.edu.cn}

\author{Mufan Zhao}
\affiliation{%
  \institution{Department of Statistics, Columbia University}
  \city{New York}
  \state{NY}
  \country{USA}
}
\email{mz3069@columbia.edu}

\author{Wenbo Zhang}
\affiliation{%
  \institution{College of Computer Science, Beijing University of Technology}
  \city{Beijing}
  \country{China}
}
\email{zhangwenbo@bjut.edu.cn}

\author{Ruijie Wang}
\authornote{Corresponding author.}
\affiliation{%
  \institution{School of Computer Science and
Engineering, Beihang University}
  \city{Beijing}
  \country{China}
}
\email{ruijiew@buaa.edu.cn}

\author{Jianxin Li}
\affiliation{%
  \institution{School of Computer Science and
Engineering, Beihang University}
  \city{Beijing}
  \country{China}
}
\email{lijx@buaa.edu.cn}
\input{text/0_abstract}




\maketitle

% \begingroup
% \small
% \noindent\raggedright\textbf{Resource Availability:}\\
% The source code of this paper has been made publicly available at
% \url{https://doi.org/10.5281/zenodo.20442900}.
% \par
% \endgroup

% \renewcommand{\shortauthors}{xxx et al.}
\renewcommand{\shortauthors}{Wang et al.}
\input{text/1_intro}
\input{text/2_background}
\input{text/3_framework}
\input{text/5_experiment}
\input{text/6_related}

\input{text/8_conclusion}

\clearpage
\bibliographystyle{ACM-Reference-Format}
% \balance
\bibliography{reference}

% \clearpage
% \newpage
\balance
\input{text/Apennew}
%\newpage
%\input{text/changes}

\end{document}

%% file: others/packages.tex
\usepackage{balance}
\usepackage{mathrsfs}
\usepackage{amsmath,amssymb,amsfonts,amsthm}
\usepackage{enumitem} 
\usepackage{graphicx}
\usepackage{extarrows}
\usepackage{amssymb}
\usepackage{verbatim}
\usepackage{makecell}
\usepackage{diagbox}
\usepackage{bm}
\usepackage{multirow}
\usepackage{float}
\usepackage{placeins}
\usepackage[linesnumbered, ruled]{algorithm2e}
\SetAlFnt{\small}
\usepackage{epstopdf}
\usepackage{setspace}
\usepackage[super]{nth}
\usepackage{fancyhdr}
\usepackage{slashbox}
\usepackage{lipsum,multicol}
\usepackage{url}
\usepackage{xcolor}
\usepackage{booktabs}
\usepackage{diagbox}
\usepackage{caption}
\usepackage{ragged2e}
\usepackage{wrapfig}
\usepackage{tikz}
\usepackage{hyperref}
\usepackage{graphicx}
\usepackage[normalem]{ulem}

%% file: others/configs.tex
\def\BibTeX{{\rm B\kern-.05em{\sc i\kern-.025em b}\kern-.08em
    T\kern-.1667em\lower.7ex\hbox{E}\kern-.125emX}}

\SetCommentSty{mycommfont}

\newtheorem{theorem}{Theorem}

\newcommand{\model}{{G$^2$LoRA}\xspace}
\allowdisplaybreaks[4]

%% file: others/authors.tex
%\author{Shengzhong Liu}
%\affiliation{
% \institution{University of Illinois at Urbana-Champaign}
% \email{}{sl29@illinois.edu}}
% 
%\author{Tianshi Wang}
%\affiliation{
% \institution{University of Illinois at Urbana-Champaign}
% \email{}{tianshi3@illinois.edu}
%}
% 
%\author{Tarek Abdelzaher}
%\affiliation{
% \institution{University of Illinois at Urbana-Champaign}
% \email{}{zaher@illinois.edu}
%}

%% file: text/0_abstract.tex
\begin{abstract}

\nop{
LLM-as-Aligner is a widely used pre-training paradigm for Text-Attributed Graphs (TAGs), aligning graph and text modalities into a shared embedding space via CLIP-style contrastive learning. While it performs well on individual downstream tasks, we observe severe catastrophic forgetting during sequential fine-tuning on streaming tasks. Though parameter-efficient fine-tuning reduces some forgetting, it doesn't fully address task interference and transfer issues. In this paper, we explore continual learning for LLM-as-Aligner methods on TAGs, aiming to reduce interference and promote positive knowledge transfer between new and old tasks. This objective presents two key challenges that limit the effectiveness of existing graph continual learning methods: (1) diverse downstream tasks lead to fluctuating learning objectives, making it hard to design a unified fine-tuning strategy; and (2) graph and text modalities have different sensitivities to fine-tuning, requiring coordinated updates. To address these, we propose Gradient Orthogonal Low-Rank Adaptation (\model) for Graph Continual Learning, which unifies optimization through a graph-text alignment loss applied to node-, link-, and graph-level tasks, enabling a single fine-tuning objective. To further reduce task interference while promoting positive knowledge transfer (Challenge 1), \model adaptively controls gradient projection at the category-level subspace, resolving task conflicts and enabling conditional backward transfer to balance forward and backward knowledge flow. To prevent misalignment and drift (Challenge 2), \model incorporates gradient magnitude modulation to coordinate update rates between graph and text encoders. Through extensive experiments, \model outperforms different backbone models across benchmark datasets, demonstrating strong performance and transferability. 
Our code and sample data are attached for anonymous review: \url{https://anonymous.4open.science/r/G2LoRA-4FF7}.
}

LLM-as-Aligner has emerged as a prevalent pre-training paradigm for Text-Attributed Graphs (TAGs), aligning graph and text modalities into a shared embedding space via CLIP-style contrastive learning. While effective on individual downstream tasks, we observe severe catastrophic forgetting when such models are sequentially fine-tuned on streaming tasks. Although parameter-efficient fine-tuning alleviates forgetting to some extent, it remains insufficient to resolve task interference and ineffective knowledge transfer.
In this work, we study graph continual learning for LLM-as-Aligner models on TAGs, with the goal of mitigating interference while promoting positive transfer across tasks. This setting introduces two fundamental challenges: (1) heterogeneous downstream tasks induce shifting optimization objectives, hindering unified fine-tuning; and (2) graph and text encoders exhibit different sensitivities to adaptation, making uncoordinated updates prone to misalignment.
To address these challenges, we propose Gradient Orthogonal Low-Rank Adaptation (\model), a continual learning framework for TAGs. \model unifies node-, link-, and graph-level tasks under a single graph–text alignment objective, and enables consistent optimization across domain/class/task incremental modes. To reduce task interference while encouraging positive transfer, \model performs category-aware gradient projection in structured subspaces, resolving conflicting updates and enabling conditional backward transfer to balance forward and backward knowledge flow. To further prevent cross-modal drift, \model introduces gradient magnitude modulation to coordinate update rates between graph and text encoders.
Extensive experiments on benchmark datasets demonstrate that \model consistently outperforms strong baselines across different backbone architectures, achieving superior continual performance and transferability. The code is available at \url{https://github.com/yuhanwang315/G2LoRA}.

\end{abstract}

%% file: text/1_intro.tex
\section{Introduction}\label{sec:intro}
%\ruijie{Let us unify model name, there is still several places using CBGP as model name (module name is okay, can change them into \model.}
Text-Attributed Graphs (TAGs) are prevalent in real-world applications such as academic citation networks, social platforms, and e-commerce systems~\cite{chen2020multi,chen2022brainnet}. In these graphs, nodes (e.g., papers, users, and products) are associated with rich textual descriptions, while edges encode their relational structure. To jointly leverage structural and semantic information, LLM-as-Aligner has emerged as a mainstream pre-training paradigm for TAGs~\cite{zhu2025graphclip,G2P2}. These methods typically follow a pre-training–fine-tuning framework, aligning graph and text representations into a shared latent semantic space via contrastive learning, which substantially improves transferability and generalization across downstream tasks.

% In practical deployments, TAG data often arrives in a streaming and incremental manner, making it infeasible to access historical data or perform joint training in a single stage~\cite{zhang2024continual}. 

In practical deployments, graph data often evolves over time, motivating studies on dynamic and continual graph learning~\cite{wang-metatkgr,wang-mpkd,wang-metahkg,zhang2024continual}. For TAGs, such evolution is coupled with textual attributes, making it infeasible to access historical data or perform joint training in a single stage. Although LLM-as-Aligner models achieve strong performance on individual downstream tasks, we observe severe catastrophic forgetting when they are sequentially fine-tuned under continual learning settings. To examine this issue, we adopt a representative LLM-as-Aligner framework, GraphCLIP~\cite{zhu2025graphclip}, as a testbed. As illustrated in Figure~\ref{fig:motivation_a}, we compare full-parameter fine-tuning, LoRA-based fine-tuning, multi-task joint training (Joint), and task-isolated training (Isolate). Full-parameter fine-tuning leads to drastic performance degradation on previous tasks, while LoRA alleviates forgetting. Nevertheless, even with parameter-efficient fine-tuning, substantial task interference and ineffective knowledge transfer persist.
% challenges and limitations of current approaches
\begin{figure}[t]
    \centering
    \captionsetup{skip=1pt}
    % --- Subfigure (a) ---
    \begin{subfigure}[t]{0.23\textwidth}
        \centering
        \includegraphics[width=0.9\linewidth]{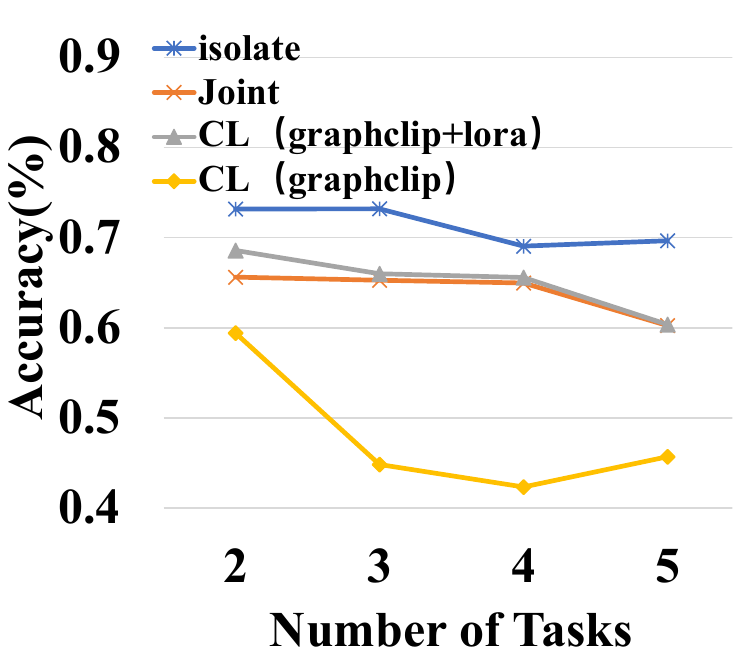}
        \subcaption{Graphclip Performance in CL}
        \label{fig:motivation_a}
    \end{subfigure}
    \hfill
    % --- Subfigure (b) ---
    \begin{subfigure}[t]{0.23\textwidth}
        \centering
        \includegraphics[width=0.9\linewidth]{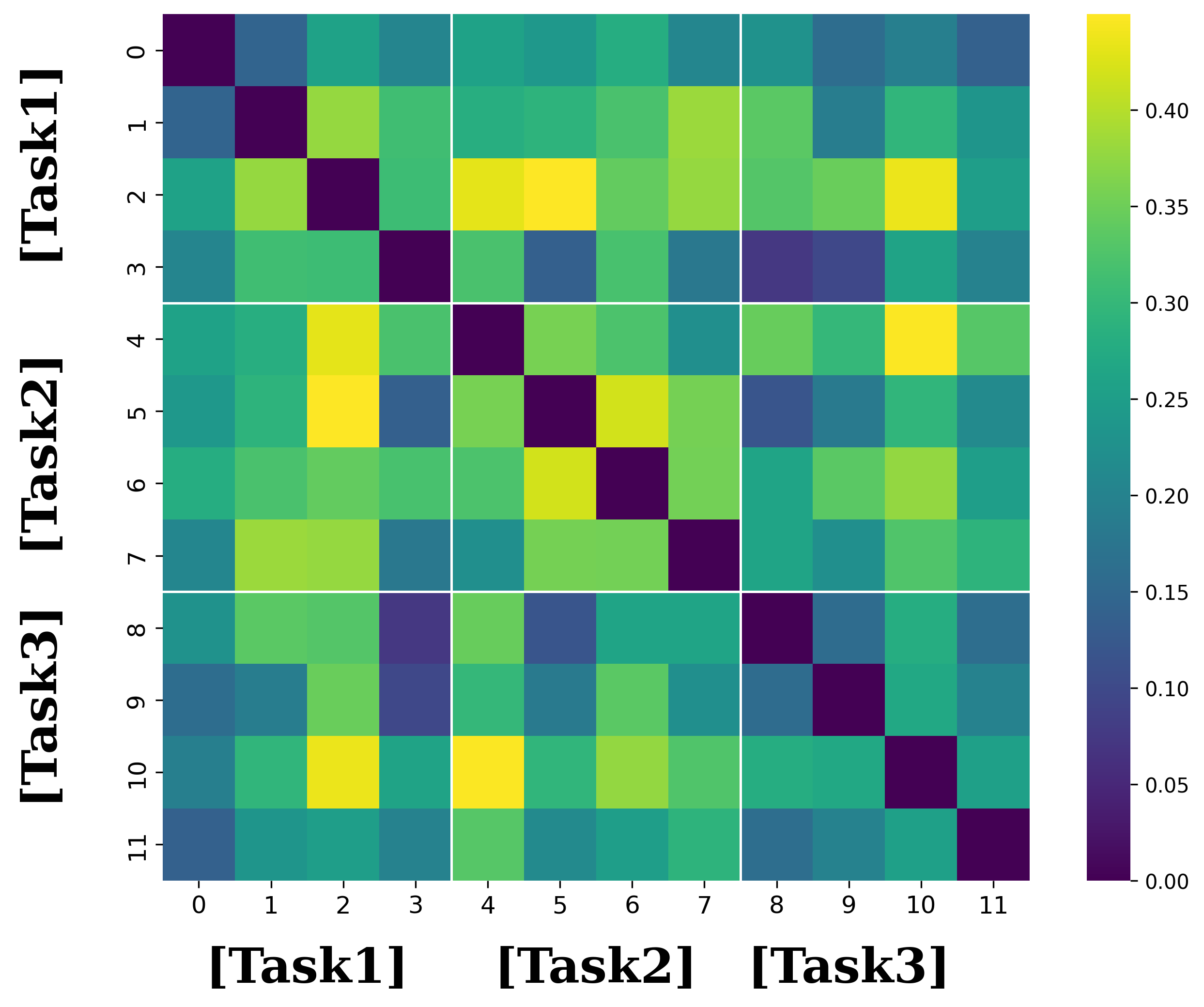}
        \subcaption{Class-Space Similarity in CIL}
        \label{fig:motivation_b}
    \end{subfigure}
    \\
    \hfill
    % --- Subfigure (c) ---
    \begin{subfigure}[t]{0.23\textwidth}
        \centering
        \includegraphics[width=0.9\linewidth]{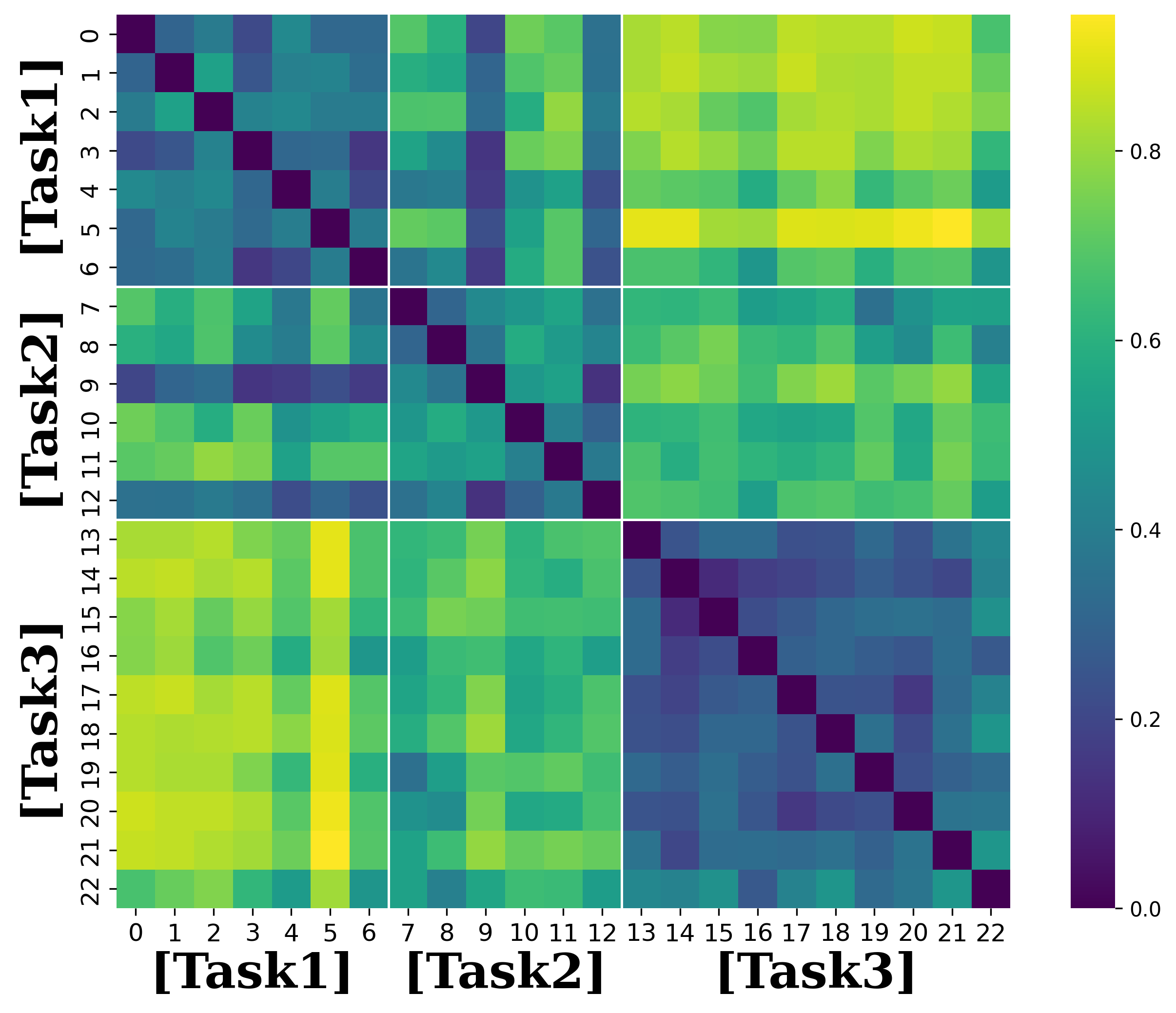}
        \subcaption{Class-Space Similarity in DIL}
        \label{fig:motivation_c}
    \end{subfigure} 
        \begin{subfigure}[t]{0.23\textwidth}
        \centering
        \includegraphics[width=0.9\linewidth]{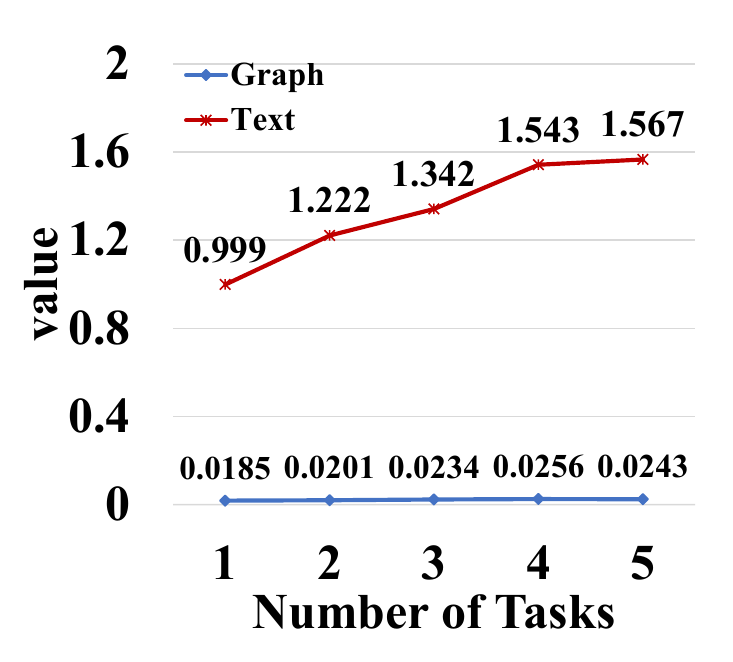}
        \subcaption{Variance of parameters}
        \label{fig:motivation_d}
    \end{subfigure}
    \caption{(a) LLM-as-Aligner Performance in CL: Joint vs. Isolate vs. CL (LoRA) vs. LoRA (b) Inter-class similarity matrix under the Class-Incremental Learning (CIL) setting, (c) Inter-class similarity matrix under the Domain-Incremental Learning (DIL) setting. (d) Parameter Changes in Graph and Text Models as Tasks Increase. 
    %\ruijie{need to add isolate mode in first figure.} 
    %\ruijie{1. figure 1(a) is a bit hard to understand and explain, as the trend differs from impression. How about we simplify it and just present catastrophic forgetting issue of LLM-as-Aligner methods, such as GraphClip and G2P2?} \ruijie{2. Figure 1d, enlarge font size, same size as subcaption.} \ruijie{3. Let us use Figure 1 to organize the introduction logic flow: 1(a) we observe forgetting issue in LLM-as-Aligner methods (GraphClip and G2P2), just use the old figure joint v.s. sequential v.s. isolate (if not too high); 1(b-c) challenge 1: diverse incremental mode and diverse task, inteference with each others; 1(d) challenge 2: cross-modal alignment.}
    }
    \label{fig:motivation}
\end{figure}

\nop{
Therefore, we investigate the continual learning problem of LLM-as-Aligner under a PEFT framework, approaching it from gradient-space relationship modeling,  with the goal of reducing task interference while promoting positive knowledge transfer between new and old tasks. However, this goal faces two key challenges that limit the applicability of existing graph continual learning methods: 
}

A key reason behind this phenomenon lies in conflicts within the gradient space, where task-specific updates overlap and interfere with previously acquired knowledge. Gradient orthogonal projection methods \cite{GPM,Inflora} have been shown to be effective in mitigating such conflicts by explicitly modeling gradient relationships. Motivated by this observation, we investigate continual learning for LLM-as-Aligner models under a PEFT framework, with the goal of reducing interference while promoting positive knowledge transfer between new and historical tasks. However, this setting introduces two fundamental challenges that limit the effectiveness of existing graph continual learning approaches.

\textbf{Challenge 1: Heterogeneous downstream tasks induce shifting optimization objectives.}
TAG downstream tasks exhibit substantial structural diversity, spanning node-level, edge-level, and graph-level predictions. This heterogeneity leads to continuously evolving learning objectives, which complicates the design of a single, parameter-efficient fine-tuning (PEFT) strategy. The challenge is greater than in traditional graph continual learning, as it requires the simultaneous optimization of both graph and text encoders. Without a shared and stable representation space, downstream tasks cannot be consistently aligned with the pre-trained graph–text embedding space, which limits the performance of pre-training models and might even cause negative transfer~\cite{sun2023all}. Moreover, task relationships in continual learning are highly non-uniform across different incremental scenarios. As shown in Figures~\ref{fig:motivation_b} and~\ref{fig:motivation_c}, task interactions may exhibit global similarity, local similarity, or block-wise similarity structures, leading to complex patterns of conflict and transfer. These observations suggest that the true units of knowledge conflict and transfer are not entire tasks, but rather the geometric relationships among category-specific discriminative subspaces. How to characterize and exploit such fine-grained subspace interactions within a unified framework, while accommodating diverse incremental settings, mitigating interference, and enabling knowledge transfer, remains an open challenge.

\nop{
\textbf{Challenge 1: }The diversity of downstream graph tasks leads to continuously evolving learning objectives, which makes it challenging to design a unified parameter-efficient fine-tuning strategy. A primary challenge stems from the structural heterogeneity across tasks at different granularities, including node-level, edge-level, and graph-level tasks. Due to the lack of a clearly shared representation space, downstream tasks cannot be stably aligned with the pre-trained task space, thereby limiting the performance of pre-training models and might even cause negative transfer~\cite{sun2023all}. Further complicating this issue are the complex and diverse relationships between tasks, which exacerbate the challenges of continual learning. Tasks in different incremental scenarios exhibit varying patterns of conflict and transfer. As illustrated in Figures~\ref{fig:motivation_b} and~\ref{fig:motivation_c}, task spaces may present global similarity, local similarity, or block-wise similarity structures, indicating highly non-uniform interactions among tasks. These observations suggest that the true units of knowledge conflict and transfer are not entire tasks themselves, but rather the geometric relationships among category-specific discriminative subspaces. Therefore, how to characterize such fine-grained subspace interactions within a unified framework, while simultaneously accommodating diverse incremental settings, mitigating task conflicts, and promoting knowledge transfer, remains a critical challenge in graph continual learning. \textbf{Challenge 2:} The graph and text modalities exhibit inherently different sensitivities to fine-tuning~\cite{peng2022balanced}, leading them to respond unevenly to the same optimization signals. During continual adaptation across multiple tasks, this discrepancy manifests as imbalanced update magnitudes between the two modalities, where the text encoder tends to undergo increasingly larger parameter changes (as illustrated in Figure~\ref{fig:motivation_d}). 
Such asynchronous updates break the synchronized evolution of the dual encoders and gradually distort the shared embedding space, resulting in cross-modal alignment drift over time. Therefore, a coordinated and modality-aware update mechanism is required to balance optimization dynamics and maintain stable cross-modal alignment.
}

\textbf{Challenge 2: Imbalanced cross-modal sensitivities cause alignment drift.}
Graph and text encoders respond differently to fine-tuning signals due to their inherent architectural and representational differences~\cite{peng2022balanced}. During continual adaptation, this discrepancy leads to imbalanced update magnitudes across modalities, with the text encoder often undergoing larger parameter changes, as illustrated in Figure~\ref{fig:motivation_d}. Such asynchronous updates disrupt the synchronized evolution of the dual encoders and gradually distort the shared embedding space, resulting in alignment drift. Therefore, maintaining stable graph–text alignment requires a coordinated, modality-aware optimization mechanism.

\nop{
To address the above challenges, we propose a graph continual learning \model framework for multiple incremental settings, combining Category-Aware Bidirectional Gradient Projection (CBGP) with methods for preserving cross-modal alignment. To tackle the first challenge, we redefine the downstream task fine-tuning approach by using a unified graph-text alignment loss, which is consistently applied to node-level, link-level, and graph-level tasks, achieving a single fine-tuning objective across all incremental modes. Furthermore, we decouple task relationships across different incremental scenarios from the category space. Specifically, we first decouple the gradient subspaces of historical tasks at the category level, allowing for a detailed representation of the discriminative structural differences between categories. Guided by class prototypes, we adaptively scale the projection strength for each subspace and impose differentiated constraints on the gradients of new tasks. This enables the unified modeling of heterogeneous task conflicts in class-incremental, domain-incremental, and task-incremental continual learning scenarios. Additionally, based on the knowledge transfer potential between task gradients, \model conditionally activates backward transfer to achieve an adaptive balance between forward and backward knowledge flow, thereby suppressing harmful interference from historical tasks while effectively leveraging transferable information between semantically related tasks.

To address the second challenge, we introduce a gradient modulation mechanism that dynamically adjusts the update rates between different modalities based on their sensitivities in continual learning. This prevents inconsistent learning dynamics between the graph and text modalities from causing representation drift, effectively alleviating cross-modal alignment shift. %\ruijie{In the end, highlight experimental results: Through extensive experiments on three incremental settings, \model~ achieves state-of-the-art performance and demonstrates robust transferability, outperforming xx baselines across xx datasets, achieving $xx\%$ relative improvements on average over the best baseline. At the same time, it is compatible with three representative LLM-as-Aligner models on TAGs.}
}
To address these challenges, we propose \model, a graph continual learning framework for LLM-as-Aligner models for three incremental settings (domain/class/task). To tackle Challenge~1, \model reformulates node-level, link-level, and graph-level tasks into a unified graph–text alignment objective, enabling a consistent fine-tuning strategy across task granularities. Building upon this formulation, \model introduces Category-Aware Bidirectional Gradient Projection (CBGP), which decomposes historical knowledge into category-level gradient subspaces. Guided by class prototypes, CBGP adaptively scales projection strengths to resolve gradient conflicts while selectively enabling conditional backward transfer, thereby balancing forward adaptation and backward knowledge consolidation across class-, domain-, and task-incremental scenarios. To address Challenge~2, \model further incorporates a gradient modulation mechanism that dynamically coordinates update rates between graph and text encoders based on their relative sensitivities during continual learning. This mechanism prevents asynchronous optimization dynamics from inducing representation drift, effectively preserving long-term cross-modal alignment.

In summary, our contributions are as follows:
\begin{itemize}[nosep, leftmargin=*]
    % \item  We propose \model, a unified graph continual learning framework that optimizes diverse downstream tasks under a shared graph-text alignment objective and characterizes geometric relationships among tasks in category-level discriminative subspaces, combined with conditional backward transfer to enable unified handling of task conflicts and knowledge transfer across different incremental settings.

    \item We propose \model, a unified graph continual learning framework for LLM-as-Aligner models that reformulates node-, link-, and graph-level tasks under a shared graph–text alignment objective, enabling consistent parameter-efficient fine-tuning across heterogeneous downstream tasks and incremental settings.
    
    \item We introduce category-aware bidirectional gradient projection, which models task conflicts and transfer through category-level discriminative subspaces and conditionally activates backward transfer, achieving unified fine-grained interference mitigation and effective knowledge transfer in class-, domain-, and task-incremental learning.

    \item We further develop a gradient magnitude modulation mechanism to coordinate update rates between graph and text encoders, stabilizing cross-modal representations and preventing alignment drift during continual adaptation.
    
    % \item we introduce a gradient magnitude modulation mechanism in the gradient orthogonal fine-tuning strategy, dynamically coordinating the update rates between graph and text encoders, thereby stabilizing cross-modal representations and preventing alignment drift.

    \item Experimental results show that \model achieves state-of-the-art performance and demonstrates strong transferability. It outperforms 18 baselines across multiple continual learning datasets for all three incremental learning settings, achieving an average relative improvement of 7.59\% over the best baseline. Additionally, it is compatible with three representative LLM-as-Aligner models on TAGs.
\end{itemize}

% % evaluation 
% We extensively evaluate the performance of the proposed framework on 
% % schedule of paper
% The rest of this paper is organized as follows:  We review the related work in Section~\ref{sec:6_related}. In Section~\ref{sec:2_background}, we describe our observations and motivations. We introduce the AutoSwitch framework in Section~\ref{sec:3_framework}, followed by the evaluation results in Section~\ref{sec:5_experiment}. 
% After discussing existing limitations in Section~\ref{sec:7_discussion}, we conclude in Section~\ref{sec:8_conclusion}.
% % Finally, we conclude the paper in Section~\ref{sec:conclusion}.

%% file: text/2_background.tex
\section{Preliminaries}\label{sec:2_background}
%\ruijie{we need to rewrite Preliminaries section, details is listed as follows. In the GCL, we need some equations and symbols, instead of pure description. I feel like the Continual Learning with LoRA can be moved to Section 3.2}

\noindent\textbf{Text-attributed Graphs (TAGs).} 
A text-attributed graph (TAG) is denoted as 
\( \mathcal{G} = (\mathcal{V}, \mathcal{E}, \mathcal{R}) \),
where \( \mathcal{V} \) is the set of \( N \) nodes, \( \mathcal{E} \) the edges, and \( \mathcal{R} \) the textual descriptions. 
Each node \( v_i \in \mathcal{V} \) is associated with a text attribute \( r_i \in \mathcal{R} \) providing semantic information.

Traditional graph learning encodes text into fixed-dimensional features, forming a node feature matrix
\( \mathbf{X} = [x_1, \ldots, x_N] \in \mathbb{R}^{N \times d} \),
where \( d \) is the embedding dimension. 
In contrast, recent LLM-as-Aligner approaches~\cite{zhu2025graphclip,G2P2} directly model text with large language models and align graph and text representations in a shared semantic space, enabling richer cross-modal interactions.

\nop{
A text-attributed graph (TAG) is defined as \( \mathcal{G} = (\mathcal{V}, \mathcal{E}, \mathcal{R}) \), where \( \mathcal{V} \) denotes the set of \( N \) nodes, \( \mathcal{E} \) represents the set of edges, and \( \mathcal{R} \) denotes the collection of textual descriptions associated with nodes. Each node \( v_i \in \mathcal{V} \) is paired with a textual attribute \( r_i \in \mathcal{R} \), which provides semantic information complementary to the graph structure.
In traditional graph learning settings, textual attributes are typically encoded into fixed-dimensional feature vectors, yielding a node feature matrix \( \mathbf{X} = [x_1, x_2, \ldots, x_N] \in \mathbb{R}^{N \times d} \), where \( d \) denotes the embedding dimension. This embedding paradigm has been widely adopted in graph neural network models. In contrast, recent LLM-as-Aligner approaches~\cite{zhu2025graphclip,G2P2} directly model textual attributes using large language models and align graph and text representations in a shared semantic space, enabling richer cross-modal interactions.
}
\begin{figure*}[t]
    \centering
    \captionsetup{skip=3pt}
    \includegraphics[width=0.95\textwidth]{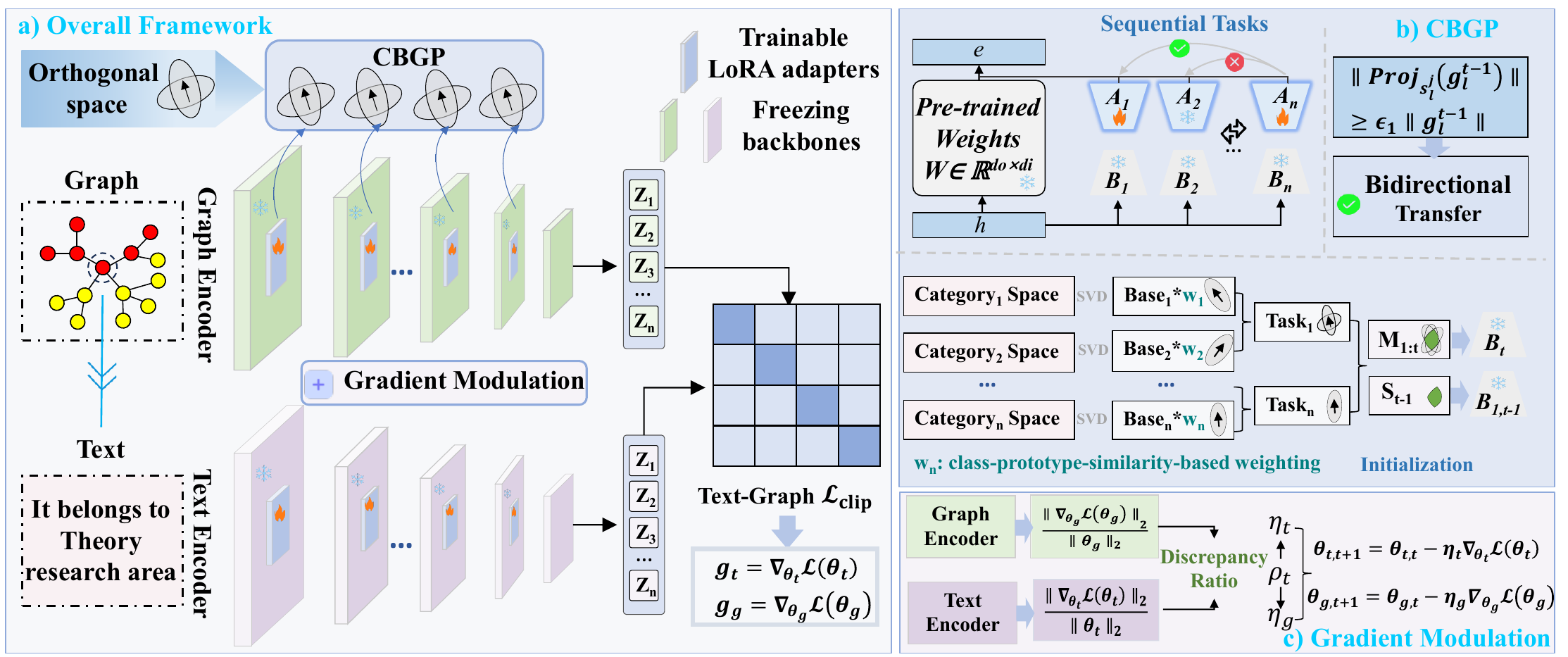}
    \caption{(a) Represents the overall framework. Graph and text encoders are frozen, and only LoRA adapters are tuned. CBGP is applied to the graph modality to constrain gradient updates and alleviate forgetting, while gradient modulation balances updates across modalities. (b) CBGP: Class-wise gradients are decomposed via SVD and weighted by class-prototype similarity to form forward/backward transfer subspaces, modeling conflicts and transfer. (c) Gradient modulation: Adjusts modality learning speeds based on gradient discrepancy for balanced updates.}
    \label{fig:framework_overview}
    \Description{1.}
\end{figure*}
% \ruijie{please refer https://arxiv.org/pdf/2505.18697 section 2.}
\noindent\textbf{Graph Continual Learning (GCL).} 
GCL addresses incrementally learning a sequence of graph-related tasks, where new tasks arrive over time, and historical data is unavailable. Newly introduced tasks often have limited labeled samples, making it challenging to match the data scale of earlier tasks. Thus, we adopt a strict $k$-shot learning paradigm for incremental tasks.
Formally, we consider an ordered sequence of training tasks
\( \{\mathcal{T}_0, \mathcal{T}_1, \dots, \mathcal{T}_{n-1}\} \),
where \( \mathcal{T}_0 \) denotes the base task and each subsequent task
\( \mathcal{T}_i \) (\( i \in \{1, 2, \dots, n-1\} \)) corresponds to an incremental task. Each task is defined as
\( \mathcal{T}_i = (\mathcal{G}_i, \mathcal{C}_i, \mathcal{V}_{l_i}) \),
where \( \mathcal{G}_i \) represents the graph (or subgraph), \( \mathcal{C}_i \) denotes the set of categories involved in the task, and \( \mathcal{V}_{l_i} \) is the set of labeled instances. Unless otherwise specified, tasks are assumed to be disjoint, i.e.,
\( \mathcal{G}_i \cap \mathcal{G}_j = \emptyset \) and
\( \mathcal{C}_i \cap \mathcal{C}_j = \emptyset \) for all \( i \neq j \).
Under this formulation, we consider three representative incremental learning settings in GCL:

\textit{(i) Class-Incremental Learning (CIL).} 
New node categories are introduced sequentially within the same graph domain. The model is required to learn representations for newly added classes while preserving performance on previously learned ones. The category set expands over time, i.e.,
\( \mathcal{C}_0 \rightarrow \mathcal{C}_0 \cup \mathcal{C}_i \).

\textit{(ii) Domain-Incremental Learning (DIL).}
% In DIL, sequential tasks correspond to different graph domains, typically instantiated as different datasets in GCL. The prediction space is shared across domains, while graph structures, node attributes, or textual distributions may vary. The model must adapt to each new domain while retaining knowledge from previous ones. Here, \( \mathcal{G}_0 \) and \( \mathcal{G}_i \) denote graph data from different domains/datasets.
% \nop{New nodes or subgraphs from different domains arrive sequentially, 
% %while the category space remains unchanged. 
% the model must adapt to domain shifts across tasks while retaining knowledge from earlier domains. In this setting, \( \mathcal{G}_0 \) and \( \mathcal{G}_i \) correspond to different graph domains.}
New nodes or subgraphs from different graph domains arrive sequentially, the model must adapt to domain shifts while retaining prior knowledge. \( \mathcal{G}_0 \) and \( \mathcal{G}_i \) represent graph data from different domains.

\textit{(iii) Task-Incremental Learning (TIL).}  
Tasks with heterogeneous prediction granularities are presented sequentially. Each task is defined as
\( \mathcal{T}_i = (\mathcal{G}_i, \mathcal{C}_i, \mathcal{V}_i) \),
where \( \mathcal{V}_i \) denotes the prediction space. Specifically,
\( \mathcal{V}_i \in \{\mathcal{V}_N, \mathcal{V}_E, \mathcal{V}_G\} \)
corresponds to node-level classification over nodes, edge-level prediction over edges, and graph-level classification over entire graphs, respectively.

\nop{
In GCL, newly added tasks usually have only a few samples, making it difficult to collect enough samples to match the scale of the initial task's incremental tasks. Therefore, we strictly follow the \textbf{k-shot} paradigm. Given an ordered sequence of training tasks \( \{\mathcal{T}_0, \mathcal{T}_1, \dots, \mathcal{T}_{n-1}\} \), where \( \mathcal{T}_0 \) serves as the base task and subsequent \( \mathcal{T}_i \) (\( i \in \{1, 2, \dots, n-1\} \)) represent incremental tasks, each task \( \mathcal{T}_i = \{\mathcal{G}_i, \mathcal{C}_i, \mathcal{V}_{l_i}\} \), where \( \mathcal{G}_i \) is the graph structure, \( \mathcal{C}_i \) is the set of categories, and \( \mathcal{V}_{l_i} \) is the set of labeled nodes, satisfies \( \mathcal{G}_i \cap \mathcal{G}_j = \mathcal{C}_i \cap \mathcal{C}_j = \emptyset \) for all \( i \neq j \). In GCL, we handle three types of incremental learning: \textbf{Class Incremental}: New node classes are introduced in the same graph domain, and the model learns representations for these classes while maintaining performance on previous ones. The category set \( \mathcal{C}_0 \) expands to \( \mathcal{C}_0 \cup \mathcal{C}_i \) for incremental tasks. \textbf{Domain Incremental}: New nodes or subgraphs from different domains are added, and the model adapts to domain shifts while retaining knowledge from earlier domains. For each task \( \mathcal{T}_i = \{\mathcal{G}_i, \mathcal{C}_i, \mathcal{V}_{l_i}\} \), \( \mathcal{G}_0 \) and \( \mathcal{G}_i \) represent different domains. \textbf{Task-Incremental (TIL)}: Tasks are sequentially presented with heterogeneous prediction granularities, corresponding to node-, edge-, and graph-level objectives. Formally, each task is defined as $\mathcal{T}_i=(\mathcal{G}_i,\mathcal{C}_i,\mathcal{V}_i)$, where $\mathcal{V}_i$ denotes the prediction space. Specifically, $\mathcal{V}_i \in \{\mathcal{V}_N, \mathcal{V}_E, \mathcal{V}_G\}$ represents node-level classification over nodes, edge-level prediction over edges, and graph-level prediction over entire graphs, respectively.
}

% \ruijie{please refer https://arxiv.org/pdf/2505.18697 section 2. few-shot ones. Highlight three incremental modes co-exists.} 

\noindent\textbf{LLM-as-Aligner on TAGs.} 
LLM-as-Aligner methods for TAGs adopt a shared dual-encoder architecture that maps graph and text modalities into a shared embedding space. A Transformer-based text encoder and a graph neural network encoder (e.g., GCN~\cite{GCN}) are jointly trained~\cite{G2P2}. Given a node-associated document, the text encoder produces a representation
\( \mathbf{t}_i \in \mathbb{R}^d \),
while the graph encoder generates a node embedding
\( \mathbf{z}_i \in \mathbb{R}^d \).

Let
\( \mathbf{T} = [\mathbf{t}_1, \ldots, \mathbf{t}_n]^\top \)
and
\( \mathbf{Z} = [\mathbf{z}_1, \ldots, \mathbf{z}_n]^\top \)
denote the text and graph embedding matrices. To establish bidirectional alignment, embeddings are normalized row-wise to obtain
\( \tilde{\mathbf{T}} \) and \( \tilde{\mathbf{Z}} \).
The node--text similarity matrix
\( \boldsymbol{\Lambda} \in \mathbb{R}^{n \times n} \)
is computed using cosine similarity:
\begin{equation}
\boldsymbol{\Lambda}
=
\tilde{\mathbf{Z}} \, \tilde{\mathbf{T}}^{\top} \cdot \exp(\tau),
\end{equation}
where \( \tau \in \mathbb{R} \) is a learnable temperature controlling the sharpness of the similarity distribution~\cite{radford2021learning}. Following CLIP-style contrastive learning, the alignment objective applies cross-entropy loss in both row-wise and column-wise directions:
\begin{equation}
\mathcal{L}_{\text{align}}(\boldsymbol{\Lambda})
=
\frac{1}{2}
\left(
\mathrm{CE}(\boldsymbol{\Lambda}, \mathbf{y})
+
\mathrm{CE}(\boldsymbol{\Lambda}^{\top}, \mathbf{y})
\right),
\end{equation}
where
\( \mathbf{y} = (1, 2, \ldots, n)^{\top} \)
denotes the identity label vector and
\( \mathrm{CE}(\cdot, \cdot) \)
is the cross-entropy loss.

\nop{LLM-as-Aligner methods for text-attributed graphs typically adopt a dual-encoder architecture that maps graph and text modalities into a shared embedding space. Specifically, a text encoder based on Transformer architectures and a graph encoder based on graph neural networks (e.g., GCN \cite{GCN}) are jointly trained~\cite{G2P2}. Given a node-associated document, the text encoder produces a representation
\( \mathbf{t}_i \in \mathbb{R}^d \),
while the graph encoder generates a corresponding node embedding
\( \mathbf{z}_i \in \mathbb{R}^d \).

Let
\( \mathbf{T} = [\mathbf{t}_1, \ldots, \mathbf{t}_n]^\top \)
and
\( \mathbf{Z} = [\mathbf{z}_1, \ldots, \mathbf{z}_n]^\top \)
denote the text and graph embedding matrices, respectively. To establish bidirectional alignment between graph and text modalities, embeddings are first normalized row-wise to obtain
\( \tilde{\mathbf{T}} \) and \( \tilde{\mathbf{Z}} \).
The node–text similarity matrix
\( \boldsymbol{\Lambda} \in \mathbb{R}^{n \times n} \)
is then computed using cosine similarity:
\begin{equation}
\boldsymbol{\Lambda}
=
\tilde{\mathbf{Z}} \, \tilde{\mathbf{T}}^{\top} \cdot \exp(\tau),
\end{equation}
where \( \tau \in \mathbb{R} \) is a learnable temperature parameter that controls the sharpness of the similarity distribution~\cite{radford2021learning}.

Following CLIP-style contrastive learning, the alignment objective is optimized by applying cross-entropy loss in both row-wise and column-wise. The overall contrastive loss is defined as:
\begin{equation}
\mathcal{L}_{\text{align}}(\boldsymbol{\Lambda})
=
\frac{1}{2}
\left(
\mathrm{CE}(\boldsymbol{\Lambda}, \mathbf{y})
+
\mathrm{CE}(\boldsymbol{\Lambda}^{\top}, \mathbf{y})
\right),
\end{equation}
where
\( \mathbf{y} = (1, 2, \ldots, n)^{\top} \)
denotes the identity label vector for contrastive supervision, and
\( \mathrm{CE}(\cdot, \cdot) \) represents the cross-entropy loss.}

%LLM as an aligner on TAGs typically utilizes a dual-modal embedding space to process TAGs by training both a text encoder and a graph encoder \cite{G2P2}. The text encoder is based on the Transformer, generating an embedding vector for each document \( \mathbf{t}_i \in \mathbb{R}^d \), where \( \theta_T \) represents the parameters of the transformer. The graph encoder is based on Graph Convolutional Networks (GCN), generating an embedding vector for each node \( \mathbf{z}_i \in \mathbb{R}^d \), where \( \theta_G \) represents the parameters of the GCN. To establish a bidirectional mapping between the graph and text, a contrastive learning strategy is employed based on the interactions between text and nodes. First, the embedding matrices \( \mathbf{T} \) and \( \mathbf{Z} \) are row-wise \( \ell_2 \)-normalized to obtain \( \tilde{\mathbf{T}} \) and \( \tilde{\mathbf{Z}} \), respectively. Then, a node-text similarity matrix \( \Lambda_1 \in \mathbb{R}^{n \times n} \) is computed to capture pairwise cosine similarity, as follows,

\nop{
LLM-as-Aligner on TAGs typically employs a dual-modal embedding space, training both a text encoder based on the Transformer and a graph encoder based on GCN \cite{G2P2}. The text encoder generates an embedding vector \( \mathbf{t}_i \in \mathbb{R}^d \) for each document, while the graph encoder generates an embedding vector \( \mathbf{z}_i \in \mathbb{R}^d \) for each node. To establish bidirectional mapping between graph and text, a contrastive learning strategy is used, first normalizing the embedding matrices \( \mathbf{T} \) and \( \mathbf{Z} \) row-wise to obtain \( \tilde{\mathbf{T}} \) and \( \tilde{\mathbf{Z}} \), then computing the node-text similarity matrix \( \Lambda_1 \in \mathbb{R}^{n \times n} \) using cosine similarity:
\begin{equation}
\Lambda_1 = \mathbf{\tilde{Z}} \mathbf{\tilde{T}}^\top \cdot \exp(\tau),
\end{equation}
where \( \tau \in \mathbb{R} \) is a trainable temperature parameter used to scale the similarity values \cite{radford2021learning}.

For contrastive learning, row- and column-wise cross-entropy loss is applied to optimize the similarity between the text and node representations. The loss function \( \mathcal{L} \) is defined as:
\begin{equation}
\mathcal{L}(\Lambda) = \frac{1}{2} \text{CE}(\Lambda, y) + \text{CE}(\Lambda^{\top}, y),
\end{equation}
where \( y = (1, 2, \dots, n)^{\top} \) is the label vector for contrastive training, and CE denotes the cross-entropy loss.
}
%\ruijie{explain the clip loss equation. follow G2P2 equation (4).} 
%\ruijie{we also need to explain the CLIP-based LLM-as-Aligner models, as we directly introduce LoRA-based continual learning in the section 3.2.}

%% file: text/3_framework.tex
\section{Methodology} \label{sec:3_framework}
%\ruijie{the symbol format needs revision. Matrices, vectors use $\mathbf{h}$, $\mathbf{H}$; space, set, loss use $\mathcal{L}$, $\mathcal{R}$}
\nop{We now introduce \model, a parameter-efficient graph continual learning framework built upon the LLM-as-Aligner paradigm. As illustrated in Figure~\ref{fig:framework_overview}, \model operates under a frozen dual-encoder architecture, where only lightweight LoRA adapters are fine-tuned for efficient adaptation. The framework integrates two key components: category-aware bidirectional gradient projection, which explicitly models and resolves task interference in the gradient space, and a gradient modulation mechanism, which coordinates optimization dynamics between graph and text modalities to preserve cross-modal alignment. In the following subsections, we detail the unified task reformulation strategy, the \model module, and the gradient modulation mechanism, respectively.}
We introduce \model, a parameter-efficient graph continual learning framework based on the LLM-as-Aligner paradigm. As shown in Figure~\ref{fig:framework_overview}, \model uses a frozen dual-encoder architecture, fine-tuning only lightweight LoRA adapters for efficient adaptation. We first propose a unified task formulation paradigm, so as to establish a unified continual learning objective. Building upon it, \model includes two key components: CBGP, which resolves task interference, and a gradient modulation mechanism, which balances optimization dynamics across graph and text modalities. The following subsections detail the unified task reformulation, CBGP module, and gradient modulation mechanism.

\nop{
The \model, which combines \model with gradient modulation (as shown in Figure~\ref{fig:framework_overview}), is designed for few-shot graph continual learning within the LLM-as-Aligner architecture using LoRA. The details of each component are discussed in the following sections.
}
%\ruijie{for model name, avoid using G-Talign, the only model we proposed is \model. If we need to refer to the backbone alignment methods, just call it LLM-as-Aligner methods. In the main experiment part, we briefly introduce which specific model we use as a representative LLM-as-Aligner method, plus we also use GraphClip and G2P2.}
\begin{figure}[t]
    \centering
    \captionsetup{skip=1pt}
    \includegraphics[width=0.8\linewidth]{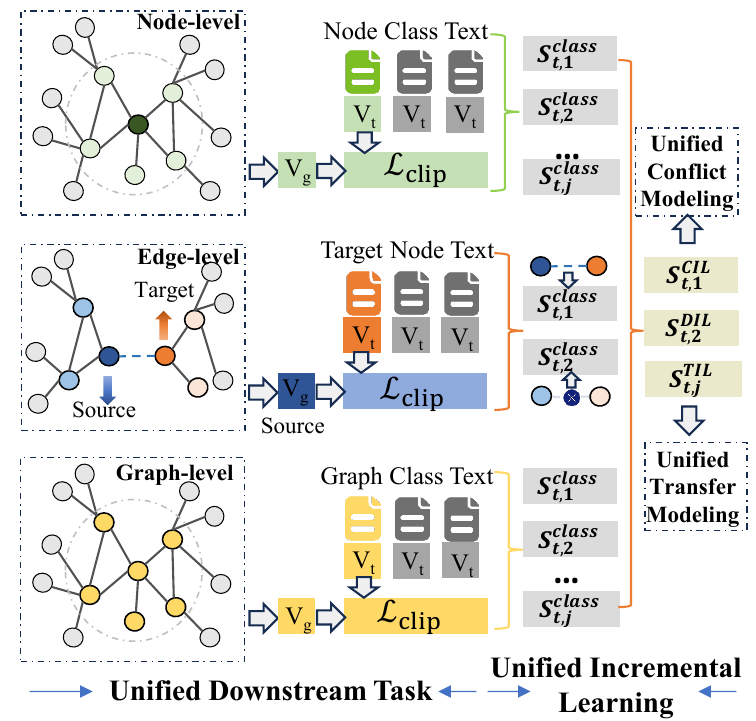}
    \caption{Left: unified downstream task. Right: unified gradient modeling for three incremental learning settings.
    }
    \label{fig:retask}
    \Description{1.}
\end{figure}

\subsection{Unified Task Formulation}
% \subsubsection{Why Reformulate Downstream Tasks.}  
% \noindent{\bf{Unified Downstream Task.}} 
Continual learning in the “pre-training and fine-tuning” framework involves various downstream tasks, such as node-, edge-, and graph-level tasks. Treating them as the same type is inappropriate due to differences in supervision signals and discriminative structures \cite{sun2023all}, which limits model performance and may cause negative knowledge transfer. Therefore, we need to reformulate these tasks into a more general form to bridge the gap.
% \subsubsection{How to Reformulate Downstream Tasks.}  
We unify task formulation by aligning graph structures with textual semantics. As shown in Figure~\ref{fig:retask}, for node-level tasks, we compute the similarity between node embeddings and corresponding textual class embeddings to obtain the similarity matrix \( \mathbf{\Lambda}_{\text{node}} \). For graph-level tasks, we compute the similarity between the entire graph embedding and its corresponding textual class embedding, yielding \( \mathbf{\Lambda}_{\text{graph}} \). For edge-level tasks, we compute the similarity between the source node's structural representation and the target node's textual embedding, resulting in \( \mathbf{\Lambda}_{\text{edge}} \). The loss functions for these tasks are unified as:
\begin{equation}
\mathcal{L}(\boldsymbol{\Lambda}) = \frac{1}{2} \left( \text{CE}(\boldsymbol{\Lambda}, \mathbf{y}) + \text{CE}(\boldsymbol{\Lambda}^\top, \mathbf{y}) \right),
\end{equation}
\noindent where $\mathbf{\Lambda}$ is the task-specific similarity matrix (e.g., $\mathbf{\Lambda}_{\text{node}}$, $\mathbf{\Lambda}_{\text{edge}}$, \( \mathbf{\Lambda}_{\text{graph}} \)), $\mathbf{y}\in (0,1)$ is the label vector for contrastive training and $\mathrm{CE}(\cdot,\cdot)$ is the cross-entropy loss function. %\ruijie{let us put all explaination right after equations in this form.}

%\ruijie{1. Let us use bold symbol for all vectors, matrices. such as $\Lambda$.}

%\ruijie{2. we explain how to unify three graph tasks, however, we have not explain why this task unification can further unify three incremental learning modes yet. }
\subsection{Category-aware Bidirectional Gradient Projection}
%\ruijie{This subsection is a bit messy, I revised a bit, still need to make the structure clear.}
In this part, we first introduce the gradient-based continual learning paradigm that forms the foundation of our design. Next, we discuss how to design a continual learning method unifying three incremental learning settings (domain/class/task incremental learning).

\noindent{\bf{Gradient Orthogonal Projection.}} GPM\cite{GPM} prevents interference between tasks by updating new tasks in the orthogonal space of the gradients of old tasks. \( \mathcal{M}_t \) represents the subspace of the model's gradients that contains the gradients of the first \( t-1 \) tasks when learning task \( t \) (where \( 1 \leq t \leq T \)). \( \mathbf{M}_t^\perp \) denotes the orthogonal complement of \( \mathbf{M}_t \). By maintaining an orthogonal basis for \( \mathbf{M}_t \) and projecting the new task gradient \( \mathbf{g}_t \) onto \( \mathbf{M}_t \), this is expressed as \( \mathbf{M}_t (\mathbf{M}_t^{\top} \mathbf{g}_t) \). Specifically, \( \mathbf{M}_t = [\mathbf{u}_1, \dots, \mathbf{u}_m] \) represents the orthogonal basis of \( \mathcal{M}_t \), where \( m = \text{dim}(\mathcal{M}_t) \). Then, GPM removes the projected gradient from \( \mathbf{g}_t \) by the following formula:
\begin{equation}
\hat{\mathbf{g}}_{t}=\mathbf{g}_{t}-\mathbf{M}_{t}(\mathbf{M}_{t})^{\top}\mathbf{g}_{t}.
\end{equation}
where \( \hat{\mathbf{g}}_t \) is the residual, lying in \( \mathcal{M}_t^\perp \). By modeling task conflicts, this method effectively prevents catastrophic forgetting. A representative work combining LoRA and gradient orthogonal strategy is InfLoRA \cite{Inflora}. In this method, before learning the \( t \)-th task, an expandable LoRA branch is introduced for each linear layer, consisting of a down-projection matrix \( \mathbf{B}_t \in \mathbb{R}^{r \times d_I} \) and an up-projection matrix \( \mathbf{A}_t \in \mathbb{R}^{d_O \times r} \). Therefore, the forward propagation of a linear layer can be expressed as:
\begin{equation}
e=\mathbf{W}\mathbf{h}+\sum_{j=1}^{\top}\mathbf{A}_j\mathbf{B}_j\mathbf{h}=\mathbf{W}_{t-1}\mathbf{h}+\mathbf{A}_t\mathbf{B}_t\mathbf{h}=\mathbf{W}_t\mathbf{h},
\end{equation}
where \( \mathbf{W}_{t}=\mathbf{W}_{t-1}+\mathbf{A}_{t}\mathbf{B}_{t}=\mathbf{W}+\sum_{i=1}^{t}\mathbf{A}_{i}\mathbf{B}_{i}.\) By pre-constructing the down-projection matrix \( \mathbf{B}_t \) in \( \mathcal{M}_t^\perp \), the update subspace of the new task is constrained, explicitly controlling the gradient interference between tasks in the parameter space.

\noindent{\bf{Category-Aware Gradient Subspace.}} Conflict and transfer differences across the three incremental settings in graph continual learning stem from structural variations in the category space. We therefore model gradient relationships uniformly in the category space. Specifically, for the \( l \)-th layer of the network, we construct a representation matrix from \( r \) samples' intermediate representations. Given the intermediate representation matrix \( \mathbf{H}_t = [x_{1,1}, x_{1,2}, \dots, x_{1,r}] \) for task \( t \), we partition it according to category labels as:
\begin{equation}
\mathbf{H}_t = [\mathbf{C}_{t,1}, \mathbf{C}_{t,2}, \dots, \mathbf{C}_{t,c_t}],
\end{equation}
where \( c_t \) denotes the number of categories in task \( t \). For each category-specific submatrix \( \mathbf{C}_{t,c} \), we process it using SVD to obtain the discriminative subspace of category \( c \), denoted as \( \mathbf{S}_{t,c} \). Repeating this for all categories in task \( t \), we obtain the task-level category-aware subspace:
\begin{equation}
\mathbf{M}_t = [\mathbf{S}_{t,1}, \mathbf{S}_{t,2}, \dots, \mathbf{S}_{t,C_t}],
\end{equation}
where \( C_t \) is the number of categories in task \( t \).

Additionally, along the temporal dimension of the continual learning process, we accumulate the task-level subspaces of all tasks to obtain the historical union subspace across tasks:
\begin{equation}
\mathbf{M}_{1:t} = [\mathbf{M}_1, \mathbf{M}_2, \dots, \mathbf{M}_t].
\end{equation}

For the three types of incremental tasks, we uniformly model the gradient space of the tasks as category space, thereby achieving unified modeling of the conflict and transfer relationships between tasks. In CIL, we construct \( \mathcal{S}_{t,c}^{\mathrm{CIL}} \) for each category of each task. In DIL, we use \( \mathcal{S}_{t,c}^{\mathrm{DIL}} \) for each category space of each domain task. In TIL, \( \mathcal{S}_{t,c}^{\mathrm{TIL}} \) corresponds to the category space of the node and the graph; for edge-level tasks, we construct a category space \( \mathcal{S}_{t,c,1}^{\mathrm{TIL}} \) for nodes with links, and a different category space \( \mathcal{S}_{t,c,2}^{\mathrm{TIL}} \) for nodes without links. 

As illustrated on the right side of Figure~\ref{fig:retask}, although the structural relationships of CIL, DIL, and TIL tasks in the gradient space differ, by using the unified modeling approach, we effectively integrate their gradient spaces into a unified category space, thus efficiently handling the conflicts and transfer relationships between the tasks.

\noindent{\bf{Dynamic Bidirectional Transfer Modeling.}}
After obtaining the class-aware historical gradient subspace 
$\mathcal{M}_{1:t-1}$, the orthonormal basis $\mathbf{H}_t$ of the current task subspace is constructed, based on which we perform bidirectional decoupling to derive the forward and backward transfer subspaces.

To adaptively control projection strength, we introduce a class-level 
modulation mechanism. Specifically, we compute the similarity $\mathrm{sim}_c$ between the 
current class subspace and previous class subspaces, and define:
\begin{equation}
\alpha_c = \zeta e^{-\mathrm{sim}_c}.
\end{equation}

According to the rank $r_j$ of each historical class-specific subspace, the scaling coefficients are expanded into a diagonal matrix
\begin{equation}
\mathbf{\Lambda}_{t-1}
=
\mathrm{diag}(
\alpha_1 \mathbf{1}_{r_1},
\dots,
\alpha_{C_{t-1}} \mathbf{1}_{r_{C_{t-1}}}
),
\end{equation}
which applies adaptive weights to the basis directions of historical gradients. With this projection, we orthogonally decouple the current gradient to obtain a non-interfering update component:
\begin{equation}
\mathbf{I}_t
=
\mathbf{H}_t
-
\mathbf{M}_{1:t-1}
\mathbf{\Lambda}_{t-1}
\mathbf{M}_{1:t-1}^{\top}
\mathbf{H}_t,
\end{equation}
whose column space lies in $\mathcal{H}_t \cap \mathcal{M}_{1:t-1}^{\perp}$. 
We use this component to initialize the LoRA down-projection matrix $\mathbf{B}_t$, thereby restricting parameter updates to a non-interfering subspace and mitigating catastrophic forgetting.

Meanwhile, we define a shared feature subspace to capture the common knowledge between the current and historical tasks, and strictly constrain the backward transfer within the common historical subspace:
\begin{equation}
\mathbf{N}_{t}
=
\sum_{i=2}^{t-1}
\left(
\mathbf{M}_{1:i-1} \cap \mathbf{M}_i
\right).
\end{equation}

Projecting the current task representation onto this shared basis yields the backward transfer component:
\begin{equation}
\mathbf{R}_t
=
\mathbf{N}_{t}
\mathbf{N}_{t}^{\top}
\mathbf{H}_t,
\end{equation}
which captures gradient directions consistent with previously learned knowledge. 
This component is used to initialize LoRA branches associated with historical tasks, thereby consolidating learned representations while preserving task-specific characteristics and enabling effective backward knowledge transfer.

To avoid negative transfer from bidirectional updates, we introduce a condition to dynamically activate backward transfer. When the gradient of task \( t \) aligns with the input subspace of historical task \( j \), backward transfer is enabled. This condition is defined as:
\begin{equation}
\| \text{Proj}_{S_j}(\nabla \mathcal{L}_t(W^{t-1})) \|_2 \geq \epsilon_1 \| \nabla \mathcal{L}_t(W^{t-1}) \|_2.
\end{equation}
When this condition is satisfied, bidirectional updates can effectively share knowledge from historical tasks,promoting their performance.

Therefore, we consider two alternative gradient update rules. For the \textbf{Rule 1}, we define the unidirectional update as:
\begin{equation}
\mathbf{W}^c = \mathbf{W} - \alpha \tilde{\mathbf{g}}_2(\mathbf{W}),
\end{equation}
and the bidirectional update \textbf{(Rule 2)} as:
\begin{equation}
\mathbf{W}^k = \mathbf{W} - \alpha \tilde{\mathbf{g}}_2(\mathbf{W}) - \alpha \tilde{\tilde{\mathbf{g}}}_2(\mathbf{W}).
\end{equation}

When the above condition holds, bidirectional updates are more effective than unidirectional updates at facilitating knowledge transfer. The following theorem provides theoretical guarantees.

%\ruijie{1. seems rule 1 and rule 2 are not specified. Need to associate with update rules above. 2. model parameter $w$ is a matrix, should use $\mathbf{W}$, please change in all places.}
% \textbf{Theorem 1.}
\begin{theorem}
\label{theo:rule}
Suppose the loss function $\mathcal{L}$ is $B$-Lipschitz and $\tfrac{H}{2}$-smooth.
Let $\mathbf{W}^c$ and $\mathbf{W}^k$ denote the model parameters obtained by applying
Rule~1 and Rule~2, respectively. If
\begin{equation}
\alpha < \frac{H}{4},
\quad
\epsilon_1 \ge \sqrt{\frac{2\alpha H}{4+\alpha H}},
\end{equation}
then
\begin{equation}
\mathcal{F}(\mathbf{W}^k) \le \mathcal{F}(\mathbf{W}^c),
\end{equation}
indicating that the bidirectional update (Rule~2) achieves a lower loss
than the unidirectional update (Rule~1). 
\end{theorem}
\begin{proof}
The proof of Theorem~\ref{theo:rule} can be found in Appendix \ref{pr_theo1}.
\end{proof}

The proposed Unified Gradient Modeling for Different Incremental Learning Scenarios, together with Category-aware Subspace Construction and Dynamic Bidirectional Transfer Modeling, jointly constitutes the core components of our proposed \model.

\subsection{Gradient Modulation}
In the dual-encoder architecture, the model must maintain alignment between the graph and text modalities during continual fine-tuning. However, as the model adapts to different downstream tasks, the two modalities often update at different rates, which can disrupt the alignment structure and lead to modality dominance or alignment drift. The work \cite{wu2022characterizing} introduces the concept of \emph{conditional learning speed}, which characterizes the relative rate at which different modalities absorb information from data by measuring the magnitude of their gradient updates, and suppresses the overly fast-learning modality via asymmetric feature scaling in the fusion layer. Inspired by this, we combine the gradient magnitude modulation mechanism with cross-modal continual learning, dynamically monitoring and balancing the update strengths of the two modalities to stabilize cross-modal alignment.

In the dual-encoder architecture, let the parameters of the graph encoder and the text encoder be denoted by \( \theta_g \) and \( \theta_t \), respectively. At the \( t \)-th iteration, given the current mini-batch loss \( \mathcal{L}(\theta_g) \) and \( \mathcal{L}(\theta_t) \), we compute the gradient norms of the two modalities as follows: 
%\ruijie{we may need to replace the gradient norms $G$, as it also represents the graph.}
\begin{equation}
\left\|g_g\right\|_2 = \left\| \nabla_{\theta_g} \mathcal{L}(\theta_g) \right\|_2, 
\qquad 
\left\|g_t\right\|_2= \left\| \nabla_{\theta_t} \mathcal{L}(\theta_t) \right\|_2.
\end{equation}

To eliminate parameter scale differences, we normalize the gradients of the graph and text encoders by dividing the gradient magnitudes by the corresponding parameter magnitudes, obtaining \( \|\widetilde{g}_g\|_2 \) and \( \|\widetilde{g}_t\|_2 \).

Based on these normalized gradients, we define the modality-wise learning speed ratio \( \rho_t \), which characterizes the relative update rate between the graph modality and the text modality at the current iteration:
\begin{equation}
\rho_t = \frac{\left\|\widetilde{g}_g\right\|_2}{\left\|\widetilde{g}_t\right\|_2 }.
\end{equation}

When \( \rho_t < 1 \), the text modality is updated faster; when \( \rho_t > 1 \), the graph modality is updated faster. Accordingly, we adaptively modulate the learning rates by applying a decay factor to the faster-learning modality. For example, when \( \rho_t < 1 \), the learning rate of the text encoder is adjusted as:
\begin{equation}
\eta_t = \eta_0 \cdot \lambda \cdot \rho_t,
\end{equation}
and the other modality is treated symmetrically when \( \rho_t > 1 \).

%Through this gradient-magnitude-based dynamic learning rate modulation, the learning speeds of the two modalities can be adaptively balanced during training, preventing any single modality from dominating the optimization process. This mechanism stabilizes cross-modal representations and effectively mitigates alignment drift in continual learning.

%
\input{table/main_result.tex}
\subsection{Category Filtering Module in CIL}
\vspace{-1mm}
In CIL, as the number of categories increases, model performance degrades. Graph data's domain knowledge and topology make task ID prediction easier. For example, TPP \cite{TPP} uses Laplacian smoothing for task ID prediction and trains independent parameters for each task. In our approach, we rely on category prototypes for classification and filter out less relevant prototypes to reduce interference from irrelevant categories. Task categories are predicted by calculating cosine similarity between prototypes. During inference, when tasks exhibit similar predictions, a top-k strategy is used to select the most relevant categories, reducing interference.

%% file: table/main_result.tex
\begin{table*}[t]
\centering
\caption{10-shot performance comparison across three incremental learning scenarios (AA/AF) (mean accuracy ± std. dev.).}
\label{tab:main_results}
\resizebox{\textwidth}{!}{
\begin{tabular}{c|cccccc|cc|cc}
\toprule
\multirow{2}{*}{Method} 
& \multicolumn{6}{c}{CIL} 
& \multicolumn{2}{|c}{DIL} 
& \multicolumn{2}{|c}{TIL} \\
\cmidrule(lr){2-7}\cmidrule(lr){8-9}\cmidrule(lr){10-11}
& \multicolumn{2}{c}{Photo} 
& \multicolumn{2}{c}{Computer} 
& \multicolumn{2}{c}{History} 
& \multicolumn{2}{|c}{Cora+Citeseer+WikiCS} 
& \multicolumn{2}{|c}{Cora+WikiCS+Photo} \\
\cmidrule(lr){2-3}\cmidrule(lr){4-5}\cmidrule(lr){6-7}\cmidrule(lr){8-9}\cmidrule(lr){10-11}
& AA$\uparrow$ & AF$\uparrow$ & AA$\uparrow$ & AF$\uparrow$  & AA$\uparrow$ & AF$\uparrow$ & AA$\uparrow$ & AF$\uparrow$  & AA$\uparrow$ & AF$\uparrow$  \\
\midrule
GCN & 43.92±1.84 & -26.44±6.27 & 43.43±0.86 & -48.52±2.94 & 46.04±1.06 & -35.23±1.33 & 43.45±2.01 & -31.45±3.87 & 59.55±0.27 & -15.92±0.47 \\
EWC    & 43.48±2.14 & -44.87±4.03 & 40.52±0.21 & -54.23±0.47 & 35.01±1.34 & -38.33±2.37 & 49.83±0.68 & -44.95±1.80 & 60.55±0.93 & -5.27±1.29 \\
LwF & 43.24±1.91 & -31.24±3.62 & 39.33±1.58 & -49.93±2.89 & 40.63±0.99 & -39.35±1.91 & 38.56±3.09 & -29.70±1.71 & 53.03±0.81 & -3.12±4.18 \\
Cosine & 52.56±1.07 & -34.12±0.73 & 45.78±0.52 & -22.01±0.91 & 51.61±0.71 & -18.36±1.14 & 54.33±0.89 & -20.31±1.63 & 62.22±1.06 & 0.00±0.00 \\
TEEN & 43.05±0.29 & -51.52±0.43 & 35.68±0.44 & -58.52±0.33 & 35.53±0.51 & -55.71±3.12 & 42.93±0.21 & -49.36±0.48 & 52.77±2.50 & -12.78±3.40 \\
TPP  & 59.78±1.87 & 0.00±0.00 & 45.68±1.63 & 0.00±0.00 & 47.95±1.18 & 0.00±0.00 & 63.75±7.90	&0.00±0.00 & 56.71±0.49 & 0.00±0.00 \\
\midrule
BERT & 38.03±0.85 & -17.42±1.31 & 35.43±0.36 & -26.51±1.11 & 30.23±0.94 & -18.53±1.53 & 29.23±1.43 & -9.38±2.73 & 35.75±3.84 & -0.22±0.66 \\
RoBERTa & 48.18±9.37 & -1.66±5.34 & 42.63±5.53 & -5.39±1.58 & 42.89±2.64 & -7.11±3.29 & 44.83±2.24 & -12.02±2.38 & 45.54±5.35 & -4.63±1.52 \\
LLaMA & 52.89±2.45 & -48.18±7.96 & 44.61±0.91 & -61.30±5.20 & 43.83±2.21 & -54.18±6.82 & 44.65±3.05 & -55.02±2.94 & 60.23±1.81 & -5.37±1.93 \\
SimpleCIL & 48.83±0.92 & -31.61±1.98 & 43.68±0.58 & -27.13±1.61 & 39.55±2.71 & -16.03±3.89 & 24.45±4.48 & -21.72±1.38 & 42.95±3.46 & -20.37±12.99 \\
\midrule
GCN$_{\text{LLMEmb}}$ & 41.03±0.53 & -53.52±0.96 & 32.26±0.37 & -62.54±0.34 & 29.14±0.75 & -50.91±2.56 & 36.15±1.85 & -58.53±23.63 & 42.83±2.53 & -8.02±5.15 \\
GraphPrompter & 55.54±5.95 & -39.23±9.82 & 40.48±1.28 & -54.03±1.43 & 43.51±2.03 & -51.22±6.13 & 48.24±1.67 & -46.52±2.47 & 51.44±2.65 & -2.66±1.83 \\
GraphGPT & 39.49±3.17 & -33.09±6.18 & 31.53±1.57 & -36.85±2.38 & 23.49±1.17 & -34.67±2.78 & 37.26±0.94 & -42.08±1.69 & 57.29±1.84 & -14.86±8.39\\
LLaGA & 39.52±4.12 & -55.03±5.83 & 35.35±1.97 & -59.43±3.48 & 30.31±1.67 & -57.61±3.12 & 38.52±1.34 & -56.18±3.12 & 57.23±0.54 & -0.07±1.31 \\
SimGCL & \underline{81.60±0.78} & -8.03±0.46 & \underline{78.64±1.06} & -6.64±0.25 & 44.55±1.05 & -7.99±0.38 & \underline{66.00±1.51} & -3.12±0.84 &60.90±1.91 & 0.09±0.30  \\
\midrule
LoRA& 50.64±3.11 & -49.79±4.76 & 57.01±0.58 & -23.67±1.49 & 44.76±1.44 & -24.11±2.11 & 65.99±0.43 & -5.98±0.56 & 73.05±1.98 &  -6.05±3.18\\
InfLoRA& 53.08±0.66 & -46.48±1.31 & 60.92±0.71 & -29.74±1.60 & \underline{53.73±0.60} & -26.80±1.08 & 65.20±0.11 & -5.88±0.63 &  \underline{74.11±4.62} &  -8.23±5.61\\
SD-LoRA& 50.42±0.67 & -39.72±2.45 & 56.59±0.28 & -25.58±0.39 & 45.08±0.79 & -25.47±2.33& 66.46±0.51 & -5.14±0.97 &  72.35±0.73 &  -9.57±1.22 \\

\midrule
\textbf{\model} 
& \textbf{84.68±0.81} 
& -2.03±1.23
& \textbf{85.65±0.11} & 0.27±0.4
& \textbf{78.81±0.85} & 4.19±0.12
& \textbf{67.74±0.07} & -3.94±0.20
& \textbf{76.18±0.33} & -2.09±0.62 \\
\midrule
%Gains(\%) & \textit{2.08} & \phantom{} & \textit{7.01} & \phantom{} & \textit{25.03} & \phantom{} & \textit{1.74} & \phantom{} & \textit{2.07} & \phantom{} \\
%\bottomrule
\end{tabular}}
\end{table*}

%% file: text/5_experiment.tex
\section{Experiments} \label{sec:5_experiment}
In this section, we empirically evaluate the proposed \model. We begin by assessing its overall performance under three representative incremental learning settings (Section~\ref{sec:4.2}). Next, we conduct ablation studies to examine the contribution of each component (Section~\ref{sec:4.3}), followed by an efficiency evaluation (Section~\ref{sec:4.5}) and an analysis of the effectiveness of the proposed gradient modulation mechanism (Section~\ref{sec:4.7}). We also analyze the model's generalization ability across different LLM-as-Aligner backbones (Section~\ref{sec:4.4}) and investigate its few-shot learning performance under varying $k$-shot regimes (Section~\ref{sec:4.6}). Finally, we provide sensitivity analyses of model hyperparameters (Section~\ref{Hyperparameter}).

\nop{
In this section, we empirically evaluate our proposed method. We first assess its overall performance in three incremental learning scenarios (Section \ref{sec:4.2}). Additionally, we perform ablation studies to validate the effectiveness of each component (Section \ref{sec:4.3}); analyze the generalization capability across different alignment frameworks (Section \ref{sec:4.4}) and conduct an efficiency analysis (Section \ref{sec:4.5}). We further explore the performance of few-shot learning under different k-shot settings (Section \ref{sec:4.6}); and verify the effectiveness of the GM mechanism (Section \ref{sec:4.7}). The sensitivity analysis of the model parameters is provided in Appendix \ref{Hyperparameter}.
}
\vspace{-2mm}
\subsection{Experimental Setups}
\label{sec:setup}
\vspace{-1mm}
\noindent{\bf Datasets.}
We evaluate \model on a total of 11 publicly available text-attributed graph (TAG) datasets spanning four diverse domains. Following the pre-training protocol of GraphCLIP~\cite{zhu2025graphclip}, five large-scale TAG datasets are used as source data for the pre-training stage. 
For downstream evaluation, we select six smaller-scale TAG datasets as target datasets and organize them into different incremental learning settings. These target datasets are used to assess the continual learning performance of LLM-as-Aligner models under few-shot and streaming scenarios. The detailed statistics of all datasets are provided in Appendix~\ref{DatasetDetails}.

\nop{
\subsubsection{Datasets.}
In this work, we utilize 11 open TAGs from four diverse domains. 
Following the pre-training protocol of GraphCLIP \cite{zhu2025graphclip}, we employ five large-scale TAGs as source datasets for the pre-training stage. 
In addition, we select six small-scale TAGs as target datasets and organize them into different incremental learning settings for downstream evaluation. 
The statistics of all datasets are reported in Table ~\ref{tab:datasets}. 
More detailed descriptions of these datasets can be found in Appendix~\ref{DatasetDetails}.
}

\noindent{\bf Baselines.}
We compare \model with a comprehensive set of baselines, which are grouped according to their backbone architectures and parameter adaptation strategies into three categories:
\textit{(1) GNN-based methods.}  
This group includes classical graph continual learning approaches built upon standard GNN backbones, such as Vanilla GCN~\cite{GCN}, EWC~\cite{EWC}, LwF~\cite{LWF}, Cosine, TPP~\cite{TPP}, and TEEN~\cite{TEEN}.
\textit{(2) LLM-based methods.}  
This category consists of pure language-model baselines without explicit graph modeling. We consider encoder-only models, including BERT~\cite{BERT} and RoBERTa~\cite{Roberta}, the decoder-only model LLaMA~\cite{llama}, as well as RoBERTa combined with SimpleCIL~\cite{simplegcl} for continual learning.
\textit{(3) GLM-based methods.}  
These methods integrate large language models with graph representations and can be further divided into three paradigms:
(i) \emph{LLM-as-Enhancer}, including GCNEmb~\cite{GCNemb} and ENGINE~\cite{ENG};
(ii) \emph{LLM-as-Predictor}, including GraphPrompter~\cite{graphptompter}, GraphGPT~\cite{Graphgpt}, LLaGA~\cite{llaga}, and SimGCL~\cite{simgcl};
and (iii) \emph{LLM-as-Aligner}, including GraphCLIP~\cite{zhu2025graphclip} and G2P2~\cite{G2P2}.
In addition, we include representative parameter-efficient fine-tuning (PEFT) baselines based on low-rank adaptation and its variants, including LoRA~\cite{hu2022lora}, InfLoRA~\cite{Inflora}, SD-LoRA~\cite{SD-LORA}, which are strong PEFT baselines to evaluate the effectiveness of \model under the same backbone architectures and continual learning protocols. %Details are provided in Appendix~\ref{baseline}.

\nop{
\subsubsection{Baselines.}
We categorize the compared methods according to their backbone integration and parameter adaptation strategies into three groups:
\textbf{(1) GNN-based methods.} 
This group includes classical continual learning approaches built upon standard GNN backbones, such as Vanilla GCN \cite{GCN}, EWC \cite{EWC}, LwF \cite{LWF}, Cosine, TPP \cite{TPP} and TEEN \cite{TEEN}. 
\textbf{(2) LLM-based methods.} 
This category consists of pure language model baselines, including encoder-only models such as BERT \cite{BERT} and RoBERTa \cite{Roberta}, the decoder-only model LLaMA \cite{llama}, as well as RoBERTa integrated with SimpleCIL \cite{simplegcl}. \textbf{(3) GLM-based methods.} 
These methods integrate large language models with graph models, which can be further divided into three paradigms: 
(i) \emph{LLM-as-Enhancer}, including GCNEmb \cite{GCNemb} and ENGINE \cite{ENG}; 
(ii) \emph{LLM-as-Predictor}, including GraphPrompter \cite{graphptompter}, Graphgpt \cite{Graphgpt} and LLaGA \cite{llaga};
(iii) \emph{LLM-as-Aligner}, including Graphclip \cite{zhu2025graphclip}, G2P2 \cite{G2P2}. In addition, we compare with representative parameter-efficient adaptation baselines based on low-rank updates and their variants, including LoRA \cite{hu2022lora}, InfLoRA \cite{Inflora}, SD-LoRA \cite{SD-LORA}, and BiLoRA \cite{zhu2025bilora}, which serve as strong PEFT baselines to evaluate the effectiveness of our proposed method under the same backbone and continual learning protocol. Detailed descriptions of these baselines are provided in the Appendix \ref{baseline}.
}

\noindent{\bf Evaluation Metrics.}
We thoroughly evaluate the continual learning performance using two widely adopted and commonly used metrics in graph continual learning: \emph{Average Accuracy} (AA) and \emph{Average Forgetting} (AF). AA measures the average predictive performance across all tasks observed so far, reflecting the overall learning effectiveness across all tasks. AF quantifies the degree of performance degradation on previously learned tasks, where higher values indicate better knowledge retention and less catastrophic forgetting over time.

\noindent{\bf Implementation Details.}
We follow the pre-training architecture of GraphCLIP\cite{zhu2025graphclip}, using a LLM as the aligner, with all data dimensions set to 384. The key difference is that we define the pretraining loss solely as a contrastive learning loss, resulting in a basic pretraining model. More implementation details can be found in the Appendix \ref{Setup}.

\nop{
\subsubsection{Evaluation metrics}
We adopt two commonly used evaluation metrics in continual graph learning: Average Accuracy (AA) and Average Forgetting (AF). 
A higher AA indicates better overall performance, while a higher AF implies less forgetting, i.e., stronger knowledge retention on previous tasks.

\subsubsection{Implementations.} We follow the pretraining model architecture of GraphCLIP\cite{zhu2025graphclip}, using a LLM as the aligner, with all data dimensions set to 384. The key difference is that we define the pretraining loss solely as a contrastive learning loss, resulting in a basic pretraining model. More implementation details can be found in the Appendix \ref{Setup}. 
}

\subsection{Overall performances}
\label{sec:4.2}
\vspace{-1mm}
Under the few-shot setting, we compare the proposed \model with representative training paradigms across three different incremental learning scenarios: class-incremental (CIL), domain-incremental (DIL), and task-incremental learning (TIL). All experiments are repeated five times, and the averaged results are reported in Table~\ref{tab:main_results}.
Overall, \model consistently outperforms existing approaches across different backbones and incremental settings. In particular, compared with the second-best baseline, our method achieves substantial improvements under the CIL scenario, with especially pronounced gains on the \textit{History} dataset, where the relative performance improvement ranges from 27.2\% to 55.32\%. Moreover, under the same backbone architecture, \model consistently surpasses GraphCLIP and various LoRA-based continual learning variants, demonstrating its effectiveness beyond just backbone or parameter-efficiency advantages.
Notably, on several datasets, \model achieves \emph{positive forgetting rates}, indicating not only effective retention of previously learned knowledge but also beneficial backward transfer. These results collectively demonstrate the effectiveness and robustness of the proposed method in challenging few-shot graph continual learning scenarios.

\nop{
Under the few-shot setting, we compare the proposed method with other mainstream training paradigms under three incremental scenarios, namely class-incremental, domain-incremental, and task-incremental learning. All experiments are repeated five times, and the average results are reported in Table~\ref{tab:main_results} Under the few-shot setting, we compare the proposed method with representative training paradigms under CIL, DIL, and TIL learning scenarios. All experiments are repeated five times, and the averaged results are reported in the table. From the results, our method consistently outperforms existing approaches across different backbones. In particular, compared with the second-best baseline SimGCL, our model achieves substantial improvements under the CIL setting, with especially remarkable gains on the History dataset, where the performance improvement ranges from 27.2\% to 55.32\%. Moreover, under the same backbone, our method consistently surpasses GraphCLIP and various LoRA-based continual learning variants. Notably, in several datasets, our approach even achieves positive forgetting rates, indicating not only effective knowledge retention but also beneficial backward transfer. These results collectively demonstrate the effectiveness and robustness of the proposed method in few-shot continual learning scenarios.
}
\input{table/table_gaussian_har.tex}

\subsection{Ablation Studies and Analysis}
\label{sec:4.3}
\noindent{\bf Ablation on Model Components.}
Table~\ref{tab:ablation} presents the ablation results under the few-shot continual learning setting. The full \model achieves the best performance in both Average Accuracy (AA) and Average Forgetting (AF), indicating that all components contribute to performance improvement and forgetting mitigation.Removing the \emph{conditional backward transfer} mechanism leads to accuracy degradation and increased forgetting, highlighting the importance of dynamically activating backward updates for knowledge consolidation. Disabling the \emph{gradient modulation} mechanism degrades performance stability, suggesting that coordinating gradient magnitudes across graph and text modalities is critical for preventing cross-modal alignment drift in dual-encoder architectures. Eliminating the \emph{bidirectional gradient projection} component results in performance drops, demonstrating the necessity of modeling both forward adaptation and backward knowledge transfer between old and new tasks. Finally, removing the \emph{category-aware subspace modeling} mechanism reduces accuracy and weakens backward transfer, confirming that class-level gradient subspaces are essential for capturing semantic distinctions and mitigating inter-class interference.
\begin{figure}[h]
    \centering
    % \vspace{-5pt}
    \captionsetup{skip=2pt}
    \begin{subfigure}[t]{0.47\linewidth}
        \centering
        \includegraphics[width=\linewidth]{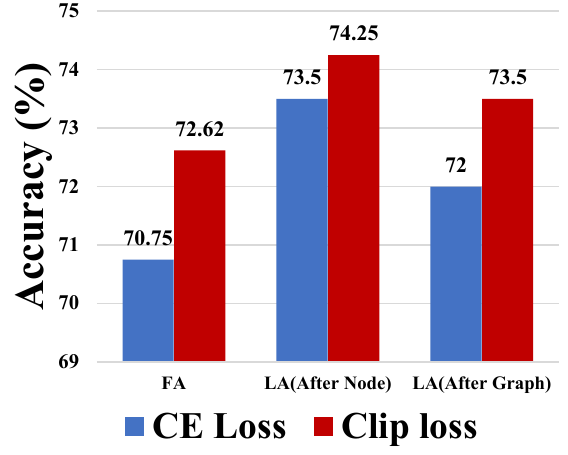}
        \caption{CE vs. Clip}
        \label{fig:wikics_link_reformulation}
    \end{subfigure}
    % --- Subfigure (a) ---
    \begin{subfigure}[t]{0.47\linewidth}
        \centering
        \captionsetup{skip=6pt}
        \includegraphics[width=\linewidth]{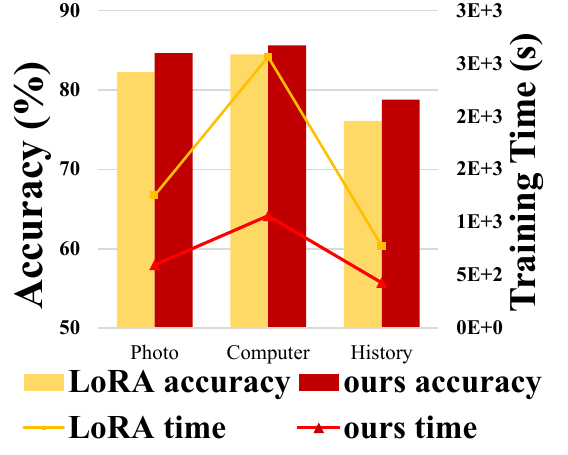} 
        \caption{Performance vs. Time}
        \label{fig:performance_vs_efficiency}
    \end{subfigure}
    \\
    % --- Subfigure (b) ---
    \begin{subfigure}[t]{0.47\linewidth}
        \centering
        \includegraphics[width=\linewidth]{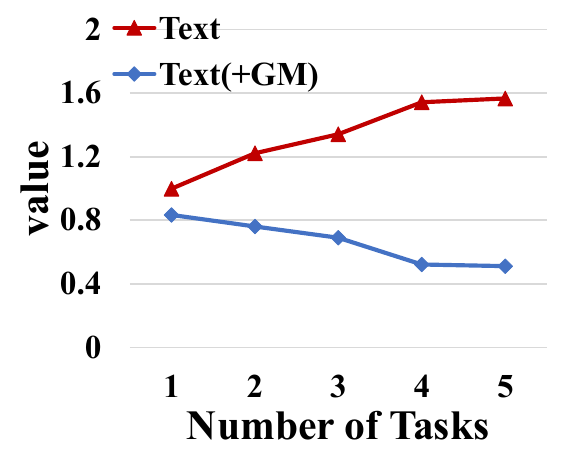}
        \caption{Parameter Variance}
        \label{fig:GM1}
    \end{subfigure}
    \begin{subfigure}[t]{0.47\linewidth}
        \centering
        \includegraphics[width=\linewidth]{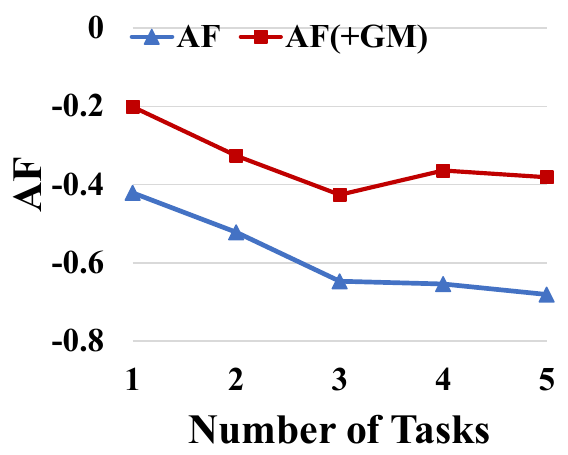}
        \caption{ Forgetting Rate on CL}
        \label{fig:GM2}
    \end{subfigure}
    \caption{(a) Performance comparison between cross-entropy loss and CLIP loss; (b) Comparison of model performance vs. training time; (c) Effect of GM on parameter variability; (d) Effect of GM on forgetting rate.}
    \label{fig:ab}
\end{figure}
\nop{
\textbf{Ablation on Model Components.}Table~\ref{tab:ablation} reports the ablation results under the few-shot continual learning setting. The full model achieves the best performance in terms of both average accuracy (AA) and average forgetting (AF), indicating that all components contribute in a complementary and synergistic manner to improving performance and alleviating forgetting. Removing the dynamic parameter activation leads to lower accuracy and more severe forgetting, suggesting that conditional backward updates are crucial for safe and effective consolidation of previously learned knowledge. Disabling the cross-modal gradient regulation further degrades stability, implying that coordinating gradient magnitudes across different modalities is important for preventing alignment drift in dual-encoder settings. Eliminating the bi-directional projection also results in performance drops, indicating the necessity of modeling bidirectional knowledge transfer between old and new tasks. Finally, removing the class-aware mechanism reduces classification accuracy and weakens backward transfer, demonstrating that class-level subspace modeling is beneficial for capturing fine-grained semantic differences and mitigating inter-class interference in class-incremental scenarios.
}
\noindent{\bf Ablation on Task Reformulation.}
Beyond component-wise ablations, we further investigate the effect of task reformulation, with a particular focus on edge-level prediction. Specifically, we compare the conventional inner-product-based link prediction optimized with cross-entropy loss against our unified graph--text alignment objective based on CLIP-style contrastive learning. Both variants are evaluated under identical encoder architectures, data splits, and training protocols to ensure fair comparison.
As shown in Figure~\ref{fig:wikics_link_reformulation}, reformulating the edge-level task using the alignment-based objective consistently improves both \emph{First ACC}, which measures the initial performance on the edge-level task, and \emph{Last ACC}, which reflects the retained performance after subsequent node-level and graph-level tasks are learned. These results indicate that alignment-based task reformulation yields more stable and transferable representations, enabling the model to better preserve edge-level knowledge under task-incremental learning. Overall, this analysis highlights the importance of unifying downstream task optimization objectives when addressing heterogeneous prediction granularities in continual learning. 

In addition, we also compare our method with GraphPrompt~\cite{liu2023graphprompt} and All in One~\cite{sun2023all}, which also focus on downstream task unification. These methods mainly unify pure graph tasks under the pre-training and fine-tuning paradigm, whereas our method targets continual graph--text learning and further addresses objective drift and task interference.
Table~\ref{tab:unified_graph_task_comparison} shows that our method achieves better AA and AF under the TIL setting.
\begin{table}[t]
\centering
\caption{Comparison with unified graph task learning methods under the TIL setting.}
\label{tab:unified_graph_task_comparison}
\small
\setlength{\tabcolsep}{12pt}
\renewcommand{\arraystretch}{1.15}
\begin{tabular}{lcc}
\toprule
\textbf{Method} & \multicolumn{2}{c}{\textbf{TIL}} \\
\cmidrule(lr){2-3}
 & \textbf{AA} & \textbf{AF} \\
\midrule
GraphPrompt & 59.13 & -17.23 \\
ALL in One & 58.31 & -14.34 \\
Ours & \textbf{76.18} & \textbf{-2.09} \\
\bottomrule
\end{tabular}
\end{table}

\nop{
\noindent \textbf{Ablation on Task Reformulation.} In addition to component-wise ablations, we further analyze the effect of task reformulation, particularly the redefinition of the link prediction task. We compare the traditional inner-product-based prediction computed with cross-entropy loss with our unified graph-text alignment CLIP loss. Both are experimented under identical encoder architectures, data splits, and training protocols. As shown in Figure~\ref{fig:wikics_link_reformulation}, reformulating the edge-level task optimization consistently improves both the First ACC (the initial accuracy of the edge-level task) and the Last ACC (the accuracy of the edge-level task after incremental learning in node-level and graph-level tasks). This indicates that alignment-based downstream task fine-tuning helps improve the model's performance in the TIL scenario.
}
\vspace{-9pt}
\input{table/GT.tex}
\subsection{Efficiency Analysis}
\label{sec:4.5}
Beyond predictive performance, we further evaluate the computational efficiency of \model in comparison with the standard LoRA baseline under the same backbone architectures and training protocols. Figure~\ref{fig:performance_vs_efficiency} illustrates the trade-off between performance and efficiency, where the bar chart reports the final accuracy and the overlaid line plot indicates the corresponding training time. In this visualization, higher accuracy and lower training time jointly reflect better efficiency.
As shown in the results, \model achieves a favorable balance between accuracy and computational cost across multiple datasets. In particular, our method consistently attains state-of-the-art predictive performance while requiring less training time than LoRA-based adaptation. These results demonstrate that the performance gains of \model do not come at the expense of increased computational overhead, but instead offer improved efficiency alongside enhanced continual learning performance.

\nop{
In addition to performance, we also evaluate the computational efficiency of our method against the baseline LoRA. Our analysis is visualized in Figure \ref{fig:performance_vs_efficiency}, where the bars represent final accuracy and the superimposed line plot represents training time. For both metrics, a higher position on the chart indicates a better result. The results reveal that our method achieves a compelling balance between high accuracy and computational efficiency. On multiple datasets, our method not only delivers state-of-the-art predictive accuracy but also exhibits significantly higher training efficiency, achieving better performance compared to the LoRA method.
}

\subsection{Effectiveness of Gradient Modulation}
\label{sec:4.7}
\nop{To verify the effectiveness of the proposed gradient modulation (GM) mechanism, we analyze parameter update dynamics and forgetting behavior over five sequential continual learning tasks. Specifically, we compare \model with and without GM under identical training settings, examining both modality-specific parameter changes and the corresponding forgetting rates.
As illustrated in Figure~\ref{fig:GM}, without GM, the text encoder exhibits progressively increasing parameter updates across tasks, indicating an imbalanced learning dynamic between modalities. In contrast, incorporating GM effectively stabilizes the update magnitude of the text modality, leading to more synchronized learning dynamics between the graph and text encoders. This stabilization directly translates into reduced forgetting rates across tasks.
These results demonstrate that GM plays a critical role in coordinating cross-modal optimization, preventing excessive adaptation in the text encoder and thereby enhancing the stability and generalization performance of \model in multi-task continual learning scenarios.}

To verify the effectiveness of the gradient modulation (GM) mechanism, we analyze parameter updates and forgetting behavior over five sequential tasks. We compare \model with and without GM under identical training conditions, examining modality-specific parameter changes and forgetting rates. As illustrated in Figure~\ref{fig:GM1} and~\ref{fig:GM2} , without GM, the text encoder exhibits progressively increasing parameter updates across tasks, indicating an imbalanced learning dynamic between modalities. In contrast, GM stabilizes the text modality's update magnitude, resulting in more synchronized learning between graph and text encoders and reduced forgetting rates. These results demonstrate that GM is crucial for coordinating cross-modal optimization, preventing excessive adaptation in the text encoder, and improving \model's stability and generalization in multi-task continual learning.
\subsection{Generalization on Alignment Framework}
\label{sec:4.4}
In this section, we evaluate the generalization of \model across different graph--text alignment frameworks. We integrate \model with representative LLM-as-Aligner approaches, including G2P2~\cite{G2P2} and GraphCLIP~\cite{zhu2025graphclip}, and assess its effectiveness under class-, domain-, and task-incremental learning scenarios. As reported in Table~\ref{tab:Alignment}, incorporating \model improves both AA and AF across alignment frameworks and incremental settings. These gains indicate enhanced performance stability and reduced catastrophic forgetting compared to the corresponding baseline models. Overall, the results demonstrate that \model is not tied to a specific alignment architecture and can be seamlessly integrated into different graph--text alignment frameworks, highlighting its generalization ability and broad applicability.

\nop{
In this section, we evaluate the generalization of our proposed method across various alignment frameworks. Specifically, we test the effectiveness of our approach when integrated with different graph-text alignment frameworks, such as G2P2 \cite{G2P2}, GraphClip \cite{G2P2}. We examine the performance improvements brought by our method in these frameworks under different incremental learning scenarios (CIL, DIL, and TIL). The results, as shown in Table~\ref{tab:Alignment}, indicate that the models incorporating our method achieve improvements in both AA and AF compared to the baseline models. These models demonstrate stronger performance stability and lower forgetting rates. This validates that our graph-text alignment method can effectively enhance the performance of different frameworks, proving its broad applicability.
}

\subsection{Few-Shot Learning with Varying k}
\label{sec:4.6}
\vspace{-1mm}
To assess performance under data-scarce conditions, we evaluate \model across few-shot settings, including 1-shot, 3-shot, 5-shot, and 10-shot regimes. Figure~\ref{fig:k-shot} reports the results under varying $k$ values. Across all shot configurations, \model outperforms the baseline SimGCL~\cite{simgcl}.
The performance advantage is pronounced in the more challenging task-incremental learning (TIL) scenario, where heterogeneous task objectives and limited supervision exacerbate catastrophic forgetting. These results indicate that \model exhibits superior few-shot generalization ability and robustness, benefiting from its unified alignment objective and structured gradient control under scarce supervision.
\begin{figure}[h]
    % \vspace{-2pt}
    \centering
    % \vspace{1pt}
    \captionsetup{skip=1pt}
    % --- Subfigure (a) ---
    \begin{subfigure}[t]{0.32\linewidth}
        \centering
        \includegraphics[width=\linewidth]{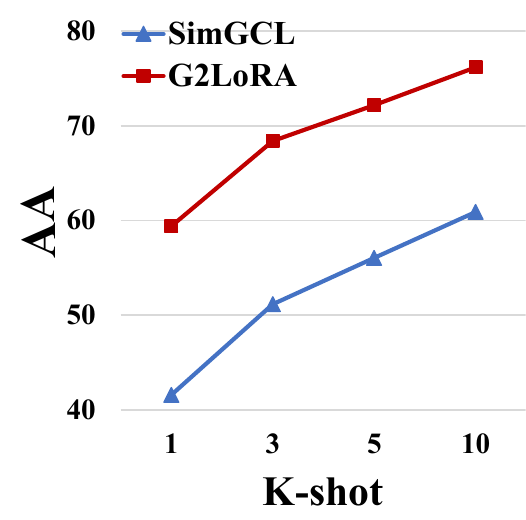} 
        \caption{\footnotesize TIL}
        \label{fig:kshota}
    \end{subfigure}
    % --- Subfigure (b) ---
    \begin{subfigure}[t]{0.32\linewidth}
        \centering
        \includegraphics[width=\linewidth]{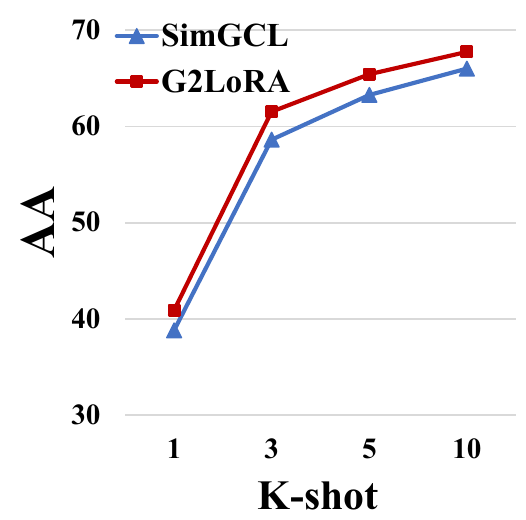}
        \caption{\footnotesize DIL}
        \label{fig:kshotb}
    \end{subfigure}
    % --- Subfigure (b) ---
    \begin{subfigure}[t]{0.32\linewidth}
        \centering
        \includegraphics[width=\linewidth]{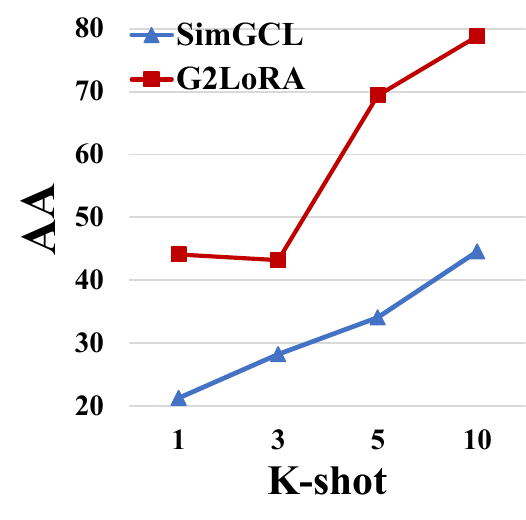}
        \caption{\footnotesize CIL}
        \label{fig:kshotc}
    \end{subfigure}
    \caption{Performance comparison of SimGCL and \model under varying K-shot configurations: (a) TIL, (b) DIL , (c) CIL.}
    \label{fig:k-shot}
\end{figure}

\subsection{Hyperparameter Sensitivity Analysis}
\label{Hyperparameter}
We analyze the sensitivity of \textbf{\model} to the bidirectional activation threshold \( \epsilon_1 \) and the regularization coefficient \( \lambda \). As shown in Figure~\ref{fig:Hyperparameter}, \textbf{\model} maintains stable performance across different values of \( \epsilon_1 \) in TIL, DIL, and CIL, indicating that the activation mechanism is robust. In CIL, a relatively larger \( \epsilon_1 \) is preferred since tasks from the same graph domain tend to produce higher projection magnitudes, enabling more effective activation of relevant old tasks. For TIL and DIL, smaller thresholds are more suitable to avoid excessive activation and preserve the balance between stability and plasticity.
\begin{figure}[t]
    \centering
    \captionsetup{skip=2pt}
    \scriptsize

    \begin{subfigure}[t]{0.31\linewidth}
        \centering
        \includegraphics[width=\linewidth]{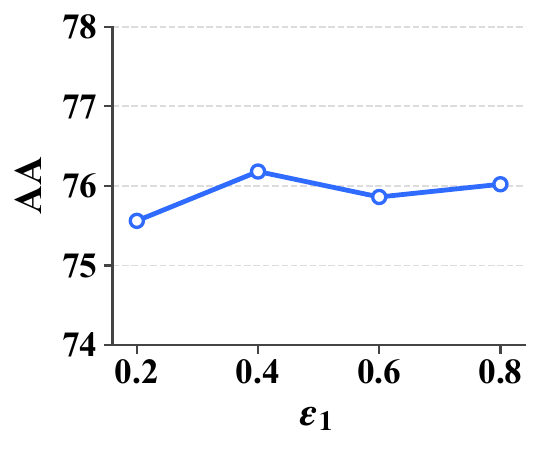}
        \caption{TIL}
        \label{fig:p_til}
    \end{subfigure}
    \hfill
    \begin{subfigure}[t]{0.31\linewidth}
        \centering
        \includegraphics[width=\linewidth]{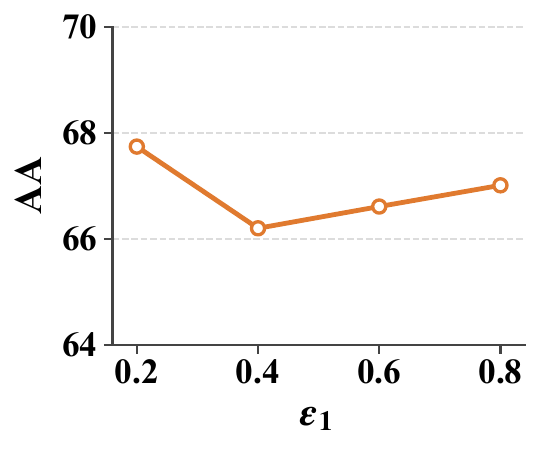}
        \caption{DIL}
        \label{fig:p_dil}
    \end{subfigure}
    \hfill
    \begin{subfigure}[t]{0.31\linewidth}
        \centering
        \includegraphics[width=\linewidth]{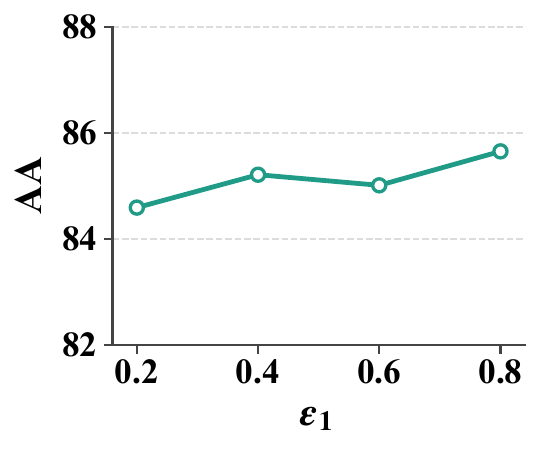}
        \caption{CIL (Computer)}
        \label{fig:p_cil}
    \end{subfigure}

    \vspace{-1mm}
    \caption{Sensitivity analysis of \( \epsilon_1 \) under different incremental learning settings.}
    \label{fig:Hyperparameter}
\end{figure}

The results for \( \lambda \) further show that \model is largely insensitive to the regularization strength. It maintains stable performance across different values and consistently outperforms the best-performing baselines, indicating that the gains are not dependent on a carefully tuned \( \lambda \). Therefore, we use \( \lambda=0.1 \) as the default setting across tasks and datasets.

\begin{table}[t]
\centering
\caption{Sensitivity analysis of \( \lambda \).}
\label{tab:lambda_sensitivity}
\begin{tabular}{lccccc}
\toprule
Dataset & \( \lambda=0.1 \) & \( \lambda=0.5 \) & \( \lambda=0.8 \) & \( \lambda=1.0 \) & Best baseline \\
\midrule
History  & 78.81 & 78.83 & 78.76 & 78.11 & 53.73 \\
Computer & 85.65 & 85.21 & 84.88 & 85.35 & 78.64 \\
\bottomrule
\end{tabular}
\end{table}

%% file: table/table_gaussian_har.tex
% Generate the table at: https://www.tablesgenerator.com/, use "Booktabs table style", center table horizontally, caption above, and scale the table to column width.

\begin{table}[t]
\caption{Ablation results on the Photo dataset under the few-shot continual learning setting (mean accuracy \%).}
\label{tab:ablation}
\centering
\small
\setlength{\tabcolsep}{6pt}
\begin{tabular}{c|cc|cc}
\toprule
\multirow{2}{*}{\textbf{Model}} & \multicolumn{2}{c|}{\textbf{CIL (Photo)}} & \multicolumn{2}{c}{\textbf{DIL}} \\
\cmidrule(lr){2-3} \cmidrule(lr){4-5}
 & \textbf{AA} & \textbf{AF} & \textbf{AA} & \textbf{AF} \\
\midrule
% SimGCL & $0.4455±0.0105$ & $-0.0799±0.0038$ &  &  \\
% GTAlign (InfLoRA) & $0.7708±0.0042$ & $0.0002±0.0086$ &  &  \\
\model(Full model) & \textbf{84.68} & \textbf{-2.03} & \textbf{67.74} & -3.94 \\
\midrule
w/o Parameter Activation & 83.41& -3.10 & 67.06 & \textbf{-3.84} \\
w/o Cross-modal Gradient & 82.71 & -3.19 & 66.48 & -5.47\\
w/o Bi-directional Projection & 83.98 & -3.25 & 67.01 & -4.10\\
w/o Category-aware & 83.60 & -3.92& 66.92 & -4.51 \\
\bottomrule
\end{tabular}
\end{table}

% \midrule
% G2P2  & & & & & & & & & & \\
% G2P2+ (ours) & & & & & & & & & & \\

%% file: table/GT.tex
\begin{table*}[t]
\centering
\caption{Performance Comparison of Graph-Text Alignment Models with and without \model Across Various Incremental Learning Scenarios(mean accuracy \%±std. dev.)}
\label{tab:Alignment}
\resizebox{\textwidth}{!}{
\begin{tabular}{c|cccccc|cc|cc}
\toprule
\multirow{2}{*}{Model}
& \multicolumn{6}{c|}{CIL} 
& \multicolumn{2}{c|}{DIL} 
& \multicolumn{2}{c}{TIL} \\
\cmidrule(lr){2-7} \cmidrule(lr){8-9} \cmidrule(lr){10-11}
& \multicolumn{2}{c}{Photo} 
& \multicolumn{2}{c}{Computer} 
& \multicolumn{2}{c}{History} 
& \multicolumn{2}{|c}{Cora+Citeseer+WikiCS} 
& \multicolumn{2}{|c}{Cora+WikiCS+Photo} \\
\cmidrule(lr){2-3}\cmidrule(lr){4-5}\cmidrule(lr){6-7}\cmidrule(lr){8-9}\cmidrule(lr){10-11}
& AA & AF & AA & AF & AA & AF & AA & AF & AA & AF \\
\midrule
GraphCLIP(+LoRA)  & 52.22±0.60 & -42.73±1.11 & 43.64±2.10  & -27.16±0.67 & 50.18±1.13 & -30.85±1.35 & 66.48±0.28 & -3.05±0.35 & 73.67±1.90  &  -6.94±1.49  \\
GraphCLIP(+\model) 
& \textbf{82.24±0.30} 
& \textbf{-1.48±0.65} 
& \textbf{84.90±0.13} 
& \textbf{-0.64±0.21}
& \textbf{76.74±0.27} 
& \textbf{-1.14±0.31} 
& \textbf{68.63±0.10} 
& \textbf{-1.30±0.20} 
& \textbf{74.65±1.64} 
& \textbf{-2.17±1.90} \\

\midrule
G2P2 & 44.65±0.12 & -65.16±1.55 & 53.19 ± 0.00 & -54.09 ± 0.52 & 49.02±0.00 &-36.03±0.76 & 31.47±1.74 & -50.83±0.54  & 35.68±0.98 &  -36.98±0.85\\
G2P2(+\model)
& \textbf{74.37±1.41} 
& \textbf{-22.80±2.57} 
& \textbf{85.30 ± 0.47} 
& \textbf{-10.73 ± 0.49} 
& \textbf{72.25±0.87} 
& \textbf{-3.61±0.34} 
& \textbf{59.37±0.73} 
& \textbf{-16.00±0.48 } 
& \textbf{72.31±1.25} 
& \textbf{-3.85±0.43} \\
\bottomrule
\end{tabular}
}
\end{table*}

%% file: text/6_related.tex
\section{Related Work} \label{sec:6_related}
\nop{
\noindent{\bf Text-Attributed Graph Methods with LLMs.}
With the rapid advancement of LLMs, recent studies have increasingly explored their integration with text-attributed graphs. Existing approaches can be broadly categorized into three paradigms: 
\textit{(i) LLM-as-Enhancer} methods leverage LLMs to enrich or encode textual node attributes, significantly improving semantic representations over traditional techniques such as Bag-of-Words. For example, TAPE~\cite{he2023explanations} utilizes ChatGPT to enhance node features.  
\textit{(ii) LLM-as-Predictor} methods reformulate graph structures and tasks into textual forms and directly perform prediction using LLMs. Representative approaches include GraphText~\cite{zhao2023graphtext}, which employs a G-Syntax Tree, as well as GraphGPT~\cite{Graphgpt} and LLaGA~\cite{llaga}, which incorporate graph neural networks for structured reasoning.  
\textit{(iii) LLM-as-Aligner} methods align graph and text representations into a shared embedding space through contrastive learning, as exemplified by G2P2~\cite{G2P2} and GraphCLIP~\cite{zhu2025graphclip}. While effective for transfer learning, these methods are primarily designed for static settings and do not explicitly address the challenges of continual graph learning.

\noindent{\bf Graph Continual Learning (GCL).}
GCL aims to enable models to learn a sequence of graph-related tasks while mitigating catastrophic forgetting and facilitating knowledge transfer. Compared with traditional continual learning, GCL faces additional challenges due to evolving graph structures and the need to preserve both semantic and topological information. Existing approaches generally fall into three categories: regularization-based methods that restrict parameter updates~\cite{liu2021overcoming,hoang2023universal}, parameter isolation methods that identify task-specific or critical parameters~\cite{zhang2023continual,zhang2022hierarchical}, and replay-based methods that store and reuse representative samples from previous tasks~\cite{zhang2023ricci,niu2024graph,liu2021overcoming}.  
Despite their effectiveness, replay-based approaches often incur significant storage and privacy overhead. In this work, we focus on replay-free continual learning strategies, which are more practical for large-scale and privacy-sensitive graph applications.

\noindent{\bf Gradient Projection Methods.}
Gradient projection techniques, such as OWM~\cite{zeng2019continual} and GPM~\cite{GPM}, mitigate catastrophic forgetting by projecting task gradients onto orthogonal subspaces, thereby preventing interference between tasks. These methods explicitly separate task-specific parameter updates and preserve previously acquired knowledge during continual learning.  
More recently, InfLoRA~\cite{Inflora} extends gradient projection to parameter-efficient fine-tuning by applying projection constraints to low-rank updates, enabling efficient adaptation with reduced forgetting. However, existing approaches typically treat tasks as atomic units and lack fine-grained control over gradient interactions. In particular, they do not account for structured relationships within the gradient space, nor do they regulate the magnitude and directional balance of parameter updates across modalities, which limits their effectiveness in complex and heterogeneous incremental learning scenarios.}

\textbf{Text-Attributed Graph Methods with LLMs.} 
%\ruijie{let us give a short related work for TAG learning with LLM. Please follow GraphClip.}
With the rapid development of LLMs, research on TAGs has gained significant attention, broadly categorized into three types: \textbf{LLM as Enhancer}, \textbf{LLM as Predictor}, and \textbf{LLM as Aligner}. \textbf{LLM as Enhancer} methods leverage LLMs to enrich or encode textual node attributes, significantly improving semantic representations over traditional techniques such as Bag-of-Words. For example, TAPE \cite{he2023explanations} uses ChatGPT to enhance node attributes. \textbf{LLM as Predictor} predicts graph data by converting it into text, e.g., GraphGPT \cite{Graphgpt} and LLaGA \cite{llaga} employ GNNs. \textbf{LLM as Aligner} aligns graph and text into a shared embedding space, with methods like G2P2 \cite{G2P2}, GraphClip \cite{zhu2025graphclip}, and ADAligner~\cite{liu2025learningnoiseresilienttransferablegraphtext}. However, these methods are generally designed for static TAG scenarios and are not specifically tailored for continual graph learning.

\noindent\textbf{Graph Continual Learning.} GCL addresses the challenges of graph data in multi-task scenarios and knowledge transfer. In addition to preventing semantic information loss, it must also tackle the forgetting of topological patterns due to evolving graph structures. Similar to traditional continual learning, GCL approaches the problem by limiting model parameter changes \cite{hoang2023universal}, isolating critical parameters for previously learned tasks \cite{zhang2023continual, zhang2022hierarchical}, and replaying representative data from previous tasks \cite{zhang2023ricci, niu2024graph}. However, data replay methods face challenges in privacy and storage. This study focuses on replay-free approaches to bypass these issues and propose more practical solutions.

\noindent \textbf{Gradient Projection Methods.} Methods such as OWM \cite{zeng2019continual} and GPM \cite{GPM}, achieve task separation by projecting gradients onto orthogonal subspaces, effectively preventing interference between tasks. These methods maintain independent parameter spaces for each task, ensuring that the learning of new tasks does not overwrite the knowledge of previously learned tasks. InfLoRA \cite{Inflora} extends this concept to PEFT by applying gradient projections to low-rank updates, enabling efficient adaptation and alleviating catastrophic forgetting. However, InfLoRA does not further consider fine-grained control of the gradient space, as well as the magnitude and direction of parameter updates, which could limit its adaptability and performance when dealing with more complex incremental tasks.

%% file: text/8_conclusion.tex
\section{Conclusion} \label{sec:8_conclusion}
\vspace{-1mm}
In this paper, we presented \model, a unified framework for continual graph--text alignment that effectively mitigates task interference under parameter-efficient adaptation. By integrating category-aware subspace modeling with bidirectional gradient projection, \model enables both stable knowledge retention and beneficial backward transfer across class-, domain-, and task-incremental learning settings. 
To further preserve cross-modal consistency and alignment stability, we introduced a dynamic gradient magnitude modulation mechanism that coordinates the optimization dynamics of graph and text encoders, preventing alignment drift during continual training. Extensive experiments on multiple benchmarks demonstrate that \model achieves state-of-the-art performance and strong transferability across different continual learning scenarios and LLM-as-Aligner backbones.

\nop{
In this paper, we presented \model, a unified framework for continual
graph–text alignment that mitigates task interference through
gradient-guided low-rank adaptation. By combining category-aware
subspace modeling with bidirectional gradient projection, our method
enables both forward knowledge preservation and backward knowledge
transfer across class-, domain-, and task-incremental settings.
We further introduced a dynamic gradient-magnitude modulation strategy
to balance the optimization dynamics between graph and text encoders,
thereby stabilizing cross-modal alignment during continual training.
Extensive experiments demonstrate that \model achieves state-of-the-art
performance and strong transferability across multiple continual learning benchmarks and LLM-as-Aligner backbones, validating the effectiveness and generality of our approach.
}

\begin{acks}
This work is partially supported by NSFC program (No. 62225202). Ruijie Wang is supported by the Fundamental Research Funds for the Central Universities.
\end{acks}

%% file: text/Apennew.tex
\appendix
\section{Appendix}
\subsection{Proof of Theorem 1}
\label{pr_theo1}
(1) For Rule 1, we have
\begin{equation}
W^c = W - \alpha \left[g_2(W) - \text{Proj}_{S_1}(g_2(W))\right] = W - \alpha \tilde{g}_2(W). 
\end{equation}

For Rule 2, we have
\begin{equation}
\small
\begin{aligned}
W^k &= W - \alpha \left[g_2(W) - \text{Proj}_{S_1}(g_2(W))\right] \\
    &\quad - \alpha \left[\text{Proj}_{S_2}(g_2(W))\right] = W - \alpha \tilde{g}_2(W) - \alpha \tilde{\tilde{g}}_2(W). 
\end{aligned}
\end{equation}

Based on the smoothness of the objective function, we can have an upper bound on \( F(W^k) \):
\vspace{-0.5mm}
\begin{equation}
\small
\begin{aligned}
\small
\mathcal{F}(\boldsymbol{W}^{k}) & \leq \mathcal{F}(\boldsymbol{W}) + \nabla \mathcal{F}(\boldsymbol{W})^{\top} (\boldsymbol{W}^{k} - \boldsymbol{W}) + \frac{H}{2} \| \boldsymbol{W}^{k} - \boldsymbol{W} \|^2 \\
& = \mathcal{F}(\boldsymbol{W}) + (\boldsymbol{g}_1(\boldsymbol{W}) + \boldsymbol{g}_2(\boldsymbol{W}))^{\top} (-\alpha \boldsymbol{\tilde{g}}_2(\boldsymbol{W}) - \alpha \boldsymbol{\tilde{\tilde{g}}}_2(\boldsymbol{W})) \\
& \quad + \frac{\alpha^2 H}{2} \| \boldsymbol{\tilde{g}}_2(\boldsymbol{W}) + \boldsymbol{\tilde{\tilde{g}}}_2(\boldsymbol{W}) \|^2 \\
& = \mathcal{F}(\boldsymbol{W}) - \alpha \langle \boldsymbol{g}_1(\boldsymbol{W}), \boldsymbol{\tilde{g}}_2(\boldsymbol{W}) \rangle - \alpha \langle \boldsymbol{g}_1(\boldsymbol{W}), \boldsymbol{\tilde{\tilde{g}}}_2(\boldsymbol{W}) \rangle \\
& \quad - \alpha \langle \boldsymbol{g}_2(\boldsymbol{W}), \boldsymbol{\tilde{g}}_2(\boldsymbol{W}) \rangle - \alpha \langle \boldsymbol{g}_2(\boldsymbol{W}), \boldsymbol{\tilde{\tilde{g}}}_2(\boldsymbol{W}) \rangle \\
& \quad + \frac{\alpha^2 H}{2} \| \boldsymbol{\tilde{g}}_2(\boldsymbol{W}) + \boldsymbol{\tilde{\tilde{g}}}_2(\boldsymbol{W}) \|^2 \\
&  \leq  \mathcal{F}(\boldsymbol{W}) - \alpha \| \boldsymbol{\tilde{\tilde{g}}}_2(\boldsymbol{W}) \|^2 - \alpha \| \boldsymbol{\tilde{g}}_2(\boldsymbol{W}) \|^2 -  \alpha \| \boldsymbol{\tilde{\tilde{g}}}_2(\boldsymbol{W}) \|^2 \\
% \alpha \langle \boldsymbol{g}_2(\boldsymbol{W}), \boldsymbol{\tilde{\tilde{g}}}_2(\boldsymbol{W}) \rangle \\
& + \frac{\alpha^2 H}{2} \|\boldsymbol{\tilde{g}}_2(\boldsymbol{W}) \|^2  + \frac{\alpha^2 H}{2} \|\boldsymbol{\tilde{\tilde{g}}}_2(\boldsymbol{W}) \|^2 \\
&  = \mathcal{F}(\boldsymbol{W}) -\left[ \alpha - \frac{\alpha^2 H}{2} \right]\|\boldsymbol{\tilde{g}}_2(\boldsymbol{W}) \|^2-\left[ 2\alpha - \frac{\alpha^2 H}{2} \right]\|\boldsymbol{\tilde{\tilde{g}}}_2(\boldsymbol{W}) \|^2 \\
\end{aligned}
\label{eq:upper_bound}
\end{equation}

and a lower bound on \( F(W^c) \)
\begin{equation}
\begin{aligned}
F(W^c) \geq F(W) + \langle \nabla F(W), W^c - W \rangle - \frac{H}{2} \| W^c - W \|^2.
\end{aligned}
\label{eq:lower_bound}
\end{equation}

Combining Eq.(\ref{eq:upper_bound}) and Eq.(\ref{eq:lower_bound}), it can be shown that:
\vspace{-0.2mm}
\begin{equation}
\small
\begin{aligned}
\mathcal{F}(\boldsymbol{W}^{k})
&\leq \mathcal{F}(\boldsymbol{W}) 
-\left[\alpha-\frac{\alpha^2H}{2}\right]\|\tilde{g}_2(W)\|^2
-\left[2\alpha-\frac{\alpha^2H}{2}\right]\|\tilde{\tilde{g}}_2(W)\|^2 \\
&\leq \mathcal{F}(\boldsymbol{W}^c)
-\langle \nabla \mathcal{F}(\boldsymbol{W}),\boldsymbol{W}^c-\boldsymbol{W}\rangle
+\frac{H}{2}\|\boldsymbol{W}^c-\boldsymbol{W}\|^2 \\
&\quad
-\left[\alpha-\frac{\alpha^2H}{2}\right]\|\tilde{g}_2(W)\|^2
-\left[2\alpha-\frac{\alpha^2H}{2}\right]\|\tilde{\tilde{g}}_2(W)\|^2 \\
&= \mathcal{F}(\boldsymbol{W}^c)
+\alpha\langle g_1(W),\tilde{g}_2(W)\rangle
+\alpha\langle g_2(W),\tilde{g}_2(W)\rangle
+\frac{\alpha^2H}{2}\|\tilde{g}_2(W)\|^2 \\
&\quad
-\left[\alpha-\frac{\alpha^2H}{2}\right]\|\tilde{g}_2(W)\|^2
-\left[2\alpha-\frac{\alpha^2H}{2}\right]\|\tilde{\tilde{g}}_2(W)\|^2 \\
&= \mathcal{F}(\boldsymbol{W}^c)
+\alpha^2H\|\tilde{g}_2(W)\|^2
-\left[2\alpha-\frac{\alpha^2H}{2}\right]\|\tilde{\tilde{g}}_2(W)\|^2 \\
&\overset{(a)}{\leq} \mathcal{F}(\boldsymbol{W}^c)
+(1-\epsilon_1^2)\alpha^2H\|g_2(W)\|^2
-\epsilon_1^2\left[2\alpha-\frac{\alpha^2H}{2}\right]\|g_2(W)\|^2 \\
&\overset{(b)}{\leq} \mathcal{F}(\boldsymbol{W}^c).
\end{aligned}
\end{equation}

\noindent where (a) holds because
\begin{equation}
\begin{aligned}
\| g_2(W) \|^2 &= \| \text{Proj}_{S_1}(g_2(W)) + \tilde{g}_2(W) \|^2 \\
&= \| \text{Proj}_{S_1}(g_2(W)) \|^2 + \| \tilde{g}_2(W) \|^2 \\
&\geq \epsilon_1^2 \| g_2(W) \|^2 + \| \tilde{g}_2(W) \|^2,
\end{aligned}
\end{equation}

\begin{equation}
\begin{aligned}
\| \text{Proj}_{S_2}(g_2(W)) \|^2 &\geq \epsilon_2^2 \| g_2(W) \|^2,
\end{aligned}
\end{equation}

where 
\begin{equation}
\epsilon_2 \leq \epsilon_1, \quad \alpha  \leq H/4
\end{equation}

and (b) is true because
\begin{equation}
\epsilon_1 \geq \sqrt{\frac{2 \alpha H}{4 + \alpha H}}.
\end{equation}

% \subsection{Experimental Setup and Baselines}
\subsection{Datasets}

\input{table/data.tex}
\label{DatasetDetails}
To obtain robust initialization for \model, we pre-train on five diverse source datasets spanning citation, biomedical, e-commerce, and social domains:

\begin{itemize}[leftmargin=8pt]
    \item \textbf{ArXiv-2023}: A directed citation network of Computer Science arXiv papers from the TAPE benchmark, where node texts are paper titles and abstracts, and labels correspond to 40 subject areas~\cite{cora}.

    \item \textbf{OGBN-ArXiv}: A large-scale MAG-indexed citation network of Computer Science papers, with textual paper features and 40 subject categories~\cite{OGBN-ArXiv}.

    \item \textbf{PubMed}: A biomedical citation network on diabetes-related publications, where node texts are scientific abstracts and labels correspond to three diabetes categories~\cite{PubMed}.

    \item \textbf{OGBN-Products}: A large Amazon co-purchasing network, where nodes are products, edges denote co-purchases, and texts are derived from product titles and descriptions~\cite{OGBN-ArXiv}.

    \item \textbf{Reddit}: A social interaction network where nodes represent users, edges indicate response relations, and node features are derived from users' historical posts~\cite{Reddit}.
\end{itemize}

For downstream incremental learning tasks, we evaluate \model on six benchmark TAGs covering citation, Wikipedia, e-commerce, and book domains:

\begin{itemize}[leftmargin=8pt]
    \item \textbf{Cora}: A citation network with 2,708 scientific publications from seven computer science categories, where edges denote citations and node texts are derived from paper titles and abstracts~\cite{cora}.
    
    \item \textbf{Citeseer}: A citation network of 3,186 scientific papers from six domains, with citation edges and textual attributes from titles and abstracts~\cite{citeseer}.
    
    \item \textbf{WikiCS}: A Wikipedia-based graph where nodes are computer science articles, edges are reference links, and node texts come from article content~\cite{wikics}.
    
    \item \textbf{Photo}: An Amazon Electronics co-purchase/co-view graph for photo equipment, where nodes are products and textual attributes are derived from reviews~\cite{photo}.
    
    \item \textbf{Computers}: An Amazon Electronics graph for computer products, with co-purchase/co-view edges and review-based textual attributes~\cite{photo}.
    
    \item \textbf{History}: An Amazon Books graph focusing on history books, where nodes are books, edges reflect co-purchase/co-view relations, and texts are derived from titles and descriptions~\cite{photo}.
\end{itemize}
% \begin{table*}[t]
% \centering
% \caption{Dataset statistics for different incremental learning settings.}
% \label{tab:datacl}
% \resizebox{0.8\textwidth}{!}{
% \begin{tabular}{lcccccccc}
% \toprule
% \multirow{2}{*}{Method} 
% & \multicolumn{1}{c}{DIL} 
% & \multicolumn{3}{c}{CIL} 
% & \multicolumn{1}{c}{TIL} \\
% \cmidrule(lr){2-2} \cmidrule(lr){3-5} \cmidrule(lr){6-6}
% & Cora+Citeseer+WikiCS
% & Photo & Computer & History 
% & Cora+WikiCS+Photo \\
% \hline
% \textbf{Classes} & \textbf{23}  & \textbf{12} & \textbf{10} & \textbf{12} & \textbf{25} \\
% \textbf{Number of Tasks} & \textbf{3} & \textbf{3} & \textbf{3} & \textbf{3} & \textbf{3} \\
% \textbf{Task Classes} & \textbf{[7,6,10]}  & \textbf{[4,4,4]} & \textbf{[4,3,3]} & \textbf{[4,4,4]} & \textbf{[7,2,12]} \\
% \hline
% \end{tabular}
% }
% \end{table*}

\textbf{Different Incremental Learning Settings.} In this Table \ref{tab:datacl}, we present the dataset statistics for different incremental learning settings, covering the DIL (Domain Incremental Learning), CIL (Class Incremental Learning), and TIL (Task Incremental Learning) scenarios. Specifically, in the DIL setting, the dataset Cora+Citeseer+WikiCS contains 23 classes and is divided into 3 tasks, with class distributions of [7, 6, 10]. In the CIL setting, the datasets include Photo, Computer, and History, with 12, 10, and 12 classes, respectively. The task class distributions are [4, 4, 4], [4, 3, 3], and [4, 4, 4]. Finally, in the TIL setting, the Cora+WikiCS+Photo dataset includes 25 classes and 3 tasks, with a task class distribution of [7, 2, 12]. Notably, in this TIL setting, Cora represents node classification, WikiCS corresponds to edge prediction, and Photo is used for graph classification. These statistics provide foundational data for evaluating the performance of incremental learning methods and highlight the class and task distributions for each learning setting.

\begin{table}[H]
\centering
\caption{Dataset statistics under different incremental learning settings.}
\label{tab:datacl}
\small
\resizebox{\columnwidth}{!}{
\begin{tabular}{lccccc}
\toprule
Setting & DIL & \multicolumn{3}{c}{CIL} & TIL \\
\cmidrule(lr){2-2} \cmidrule(lr){3-5} \cmidrule(lr){6-6}
Dataset & Cora+Cite.+WikiCS & Photo & Comp. & History & Cora+WikiCS+Photo \\
\midrule
Classes & 23 & 12 & 10 & 12 & 21 \\
Tasks & 3 & 3 & 3 & 3 & 3 \\
Task classes & [7,6,10] & [4,4,4] & [4,3,3] & [4,4,4] & [7,2,12] \\
\bottomrule
\end{tabular}
}
\end{table}

\subsection{Experimental Setup}
\label{Setup}
\textbf{Baseline Setup}.

For GNN-based methods, since they cannot process raw text directly, we first utilize a pre-trained Sentence-BERT to extract static semantic embeddings from node attributes and feed them as initial node features. We then train these models sequentially on the streaming tasks using the training set of the current session. For regularization-based (EWC) and distillation-based (LwF) baselines, we apply their respective constraints to the loss function during sequential training to mitigate forgetting.

For LLM and GLM-based methods, to eliminate performance variances caused by different foundation models, we adopt LLaMA-3-8b as the unified backbone for all generative and alignment-based baselines where applicable. For discriminative language models like BERT and RoBERTa, we perform sequential fine-tuning on the text classification task. Uniquely for SimpleCIL, consistent with \cite{simplegcl}, we adopt a frozen-backbone strategy: the pre-trained LLM is used strictly as a feature extractor, and only a lightweight projection layer with class prototypes is updated, ensuring high stability.

For fair comparisons, we execute all experiments in a strictly sequential manner without accessing future data. For optimization, we use the AdamW optimizer. To ensure convergence while preventing overfitting, we set the maximum epochs dynamically based on the model type, coupled with an early stopping mechanism. All models are implemented in PyTorch and trained on NVIDIA A100-80G GPUs.

\noindent \textbf{\model Setup}.

\textbf{(1) Pretraining.}
We initialize the \model backbone via a contrastive pre-training stage on the five source datasets described in Appendix A.2. In our main experiments, we use the AdamW optimizer with both the learning rate and weight decay set to \( 1 \times 10^{-5} \). The graph model used is GraphGPS, consisting of 12 layers with a hidden size of 1024. For the text model, we use a fine-tuned version of MiniLM 3, featuring 6 layers with a hidden size of 384. To align both models in a unified subspace, we apply a projector to transform the graph model's 1024 dimensions to match the 384 dimensions of the text model. During pretraining, we only optimize the parameters of the graph model and projector, keeping the text model frozen to reduce training costs and mitigate catastrophic forgetting. Pretraining is conducted over 30 epochs with a batch size of 800 per GPU, utilizing a single A100-80G GPU for pretraining. We will release our pretrained checkpoint after the anonymous phase.

\textbf{(2) Fine-tuning.} 
We use the AdamW optimizer and apply early stopping based on validation performance to prevent overfitting. The batch size is set to 32, and the learning rate is consistent across all baselines, with the optimal parameters obtained by searching within a fixed space. \model is added on the graph side, while the standard LoRA is used for fine-tuning the text side. The specific model parameters will be provided in our code.

\subsection{Notations and Symbols}

\begin{table}[H]
\caption{Notations and Symbols}
\label{tab:notation}
\centering
\small
\setlength{\tabcolsep}{6pt}
\resizebox{\columnwidth}{!}{
\begin{tabular}{c|l}
\toprule
\textbf{Notation} & \textbf{Description} \\
\midrule
$G = (\mathcal{V}, \mathcal{E}, \mathcal{R})$ & Text-Attributed Graph with node, edge, and text sets \\
$\{\mathcal{T}_0, \mathcal{T}_1, ..., \mathcal{T}_{n-1}\}$ & Sequence of continual learning tasks \\
$\mathcal{T}_i = \{G_i, C_i, V_{l_i}\}$ & The $i$-th task with graph, categories, and labeled nodes \\
$\theta_G, \theta_T$ & Parameters of graph and text encoders \\
$\mathbf{W}$ & Model weight matrix \\
$\mathbf{z}_i, \mathbf{t}_i$ & Node embeddings from graph and text encoders \\
$\mathbf{\Lambda}$ & Similarity matrix for graph-text alignment \\
\midrule
$\mathbf{A}_t, \mathbf{B}_t$ & LoRA up- and down-projection matrices \\
$\mathbf{H}_t$ & Intermediate representation matrix \\
$\mathbf{S}_{t,c}$ & Discriminative subspace for category $c$ \\
$\mathbf{M}_t$ & Task-level category-aware subspace \\
$\mathbf{M}_{1:t}$ & Accumulated historical union subspace \\
$\alpha_c$ & Class-wise scaling coefficient \\
$\mathbf{I}_t$ & Constrained update subspace \\
$\mathbf{R}_t$ & Backward transfer subspace \\
$\rho_t$ & Modality-wise learning speed ratio \\
\bottomrule
\end{tabular}}
\end{table}

% \subsection{ Additional Experiments and Analyses}

%% file: table/data.tex
\begin{table}[H]
\centering
\caption{Statistics of downstream datasets.}
\label{tab:datasets}
\scriptsize
\small
\resizebox{\columnwidth}{!}{
\begin{tabular}{lcccccc}
\toprule
\textbf{Dataset} & Cora & Cite. & Photo & Comp. & History & WikiCS \\
\midrule
\textbf{Nodes} & 2,708 & 3,186 & 48,362 & 87,229 & 41,551 & 11,701 \\
\textbf{Edges} & 5,429 & 4,277 & 500,928 & 721,081 & 358,574 & 215,863 \\
\textbf{Classes} & 7 & 6 & 12 & 10 & 12 & 10 \\
\textbf{Domain} & Cit. & Cit. & E-com. & E-com. & E-com. & Web \\
\bottomrule
\end{tabular}
}
\end{table}